\documentclass[conference]{IEEEtran}  

\IEEEoverridecommandlockouts                              

\pdfoutput=1

\usepackage{graphicx}
\usepackage{epsfig} 
\usepackage{color}

\usepackage{mathptmx}
\usepackage{amsmath} 

\usepackage{amssymb}  
\usepackage{amsthm}
\usepackage[mathcal]{euscript}

\usepackage{mathrsfs}
\usepackage{mathtools}

\usepackage{algpseudocode}
\usepackage[ruled,vlined,linesnumbered]{algorithm2e}

\usepackage[colorlinks=true,linkcolor=black,anchorcolor=black,citecolor=black,filecolor=black,menucolor=black,runcolor=black,urlcolor=black]{hyperref}
\usepackage{lipsum}
\usepackage{etoolbox}
\usepackage[table,xcdraw,dvipsnames]{xcolor}
\usepackage{multirow}
\usepackage{wasysym}
\usepackage{bbm}
\usepackage{cleveref}

\usepackage{outlines}
\let\labelindent\relax
\usepackage{enumitem}
\setenumerate[1]{label=\Roman*.}
\setenumerate[2]{label=\Alph*.}
\setenumerate[3]{label=\arabic*.}
\setenumerate[4]{label=\roman*.}

\usepackage{pifont}

\usepackage{booktabs}
\usepackage[footnotesize]{caption} 
\makeatletter
\patchcmd{\@makecaption}
  {\scshape}
  {}
  {}
  {}
\makeatother

\usepackage[
    style=ieee,
    doi=false,
    isbn=false,
    url=false,
    eprint=false,
    backend=biber,
    natbib=true
    ]{biblatex}
    
\bibliography{references}

\title{\LARGE \bf
Reachability-based Trajectory Design with Neural Implicit Safety Constraints}

\author{Jonathan Michaux$^{1}$, Qingyi Chen$^{2}$, Yongseok Kwon$^{3}$ and Ram Vasudevan$^{1,3}$
\thanks{This work is supported by the Ford Motor Company via the Ford-UM
Alliance under award N022977, National Science Foundation Career Award \#1751093 and by the Office of Naval Research under Award Number N00014-18-1-2575 }
\thanks{$^{1}$Robotics Institute, University of Michigan, Ann Arbor, MI $\langle$\texttt{jmichaux,,ramv, chenqy}$\rangle$\texttt{@umich.edu}.}
\thanks{$^{2}$Computer Science, University of Michigan, Ann Arbor, MI \texttt{chenqy@umich.edu.edu}.}
\thanks{$^{3}$Mechanical Engineering, University of Michigan, Ann Arbor, MI $\langle$\texttt{ramv, kwonys}$\rangle$\texttt{@umich.edu}.}
}

\begin{document}

\newif\ifcommentson
\commentsontrue
\newcounter{PatrickCount}
\addtocounter{PatrickCount}{1}
\newcommand{\pat}[1]{\textcolor{OliveGreen}{\ifcommentson\textbf{(\thePatrickCount) Patrick}: (#1)\fi}\addtocounter{PatrickCount}{1}}

\newcounter{JonCount}
\addtocounter{JonCount}{1}
\newcommand{\jon}[1]{\textcolor{RoyalBlue}{\ifcommentson\textbf{(\theJonCount) Jon}: (#1)\fi}\addtocounter{JonCount}{1}}

\newcounter{YongCount}
\addtocounter{YongCount}{1}
\newcommand{\yong}[1]{\textcolor{OliveGreen}{\ifcommentson\textbf{(\theYongCount) Yong}: (#1)\fi}\addtocounter{YongCount}{1}}

\newcounter{QingyiCount}
\addtocounter{QingyiCount}{1}
\newcommand{\qc}[1]{\textcolor{YellowOrange}{\ifcommentson\textbf{(\theQingyiCount) Qingyi}: (#1)\fi}\addtocounter{QingyiCount}{1}}

\newcounter{NotesCount}
\addtocounter{NotesCount}{1}
\newcommand{\notes}[1]{\textcolor{Blue}{\ifcommentson\textbf{(\theNotesCount) Notes}: (#1)\fi}\addtocounter{NotesCount}{1}}

\newcounter{ShreyCount}
\addtocounter{ShreyCount}{1}
\newcommand{\shrey}[1]{\textcolor{RoyalBlue}{\ifcommentson\textbf{(\theShreyCount) Shrey}: (#1)\fi}\addtocounter{ShreyCount}{1}}

\newcounter{RamCount}
\addtocounter{RamCount}{1}
\newcommand{\Ram}[1]{\textcolor{WildStrawberry}{\ifcommentson\textbf{(\theRamCount) Ram}: (#1)\fi}\addtocounter{RamCount}{1}}

\newcounter{FixCount}
\addtocounter{FixCount}{1}
\newcommand{\fix}[1]{\textcolor{Purple}{\ifcommentson\textbf{(\theFixCount) FIX}: (#1)\fi}\addtocounter{FixCount}{1}}

\newcommand{\setop}[1]{{\mathrm{\textnormal{\texttt{#1}}}}}
\newcommand{\numop}[1]{{\mathrm{\textnormal{\texttt{#1}}}}}

\newcommand{\homtrans}{H}
\newcommand{\co}{c_O}
\newcommand{\Go}{G_O}

\newcommand{\cf}{c_F}
\newcommand{\Gf}{G_F}
\newtheorem{defn}{Definition}
\newtheorem{rem}[defn]{Remark}
\newtheorem{lem}[defn]{Lemma}
\newtheorem{prop}[defn]{Proposition}
\newtheorem{assum}[defn]{Assumption}
\newtheorem{ex}[defn]{Example}
\newtheorem{thm}[defn]{Theorem}
\newtheorem{cor}[defn]{Corollary}
\newtheorem{con}[defn]{Conjecture}
\newtheorem{problem}[defn]{Problem}

\providecommand{\R}{\ensuremath \mathbb{R}}
\providecommand{\IR}{\ensuremath \mathbb{IR}}
\providecommand{\N}{\ensuremath \mathbb{N}}
\providecommand{\Q}{\ensuremath \mathbb{Q}}

\newcommand{\regtext}[1]{\mathrm{\textnormal{#1}}}
\newcommand{\ol}[1]{\overline{#1}}
\newcommand{\ul}[1]{\underline{#1}}
\newcommand{\defemph}[1]{\emph{#1}}
\newcommand{\ts}[1]{\textsuperscript{#1}}

\newcommand{\comp}{^{\regtext{C}}}
\newcommand{\card}[1]{\left\vert#1\right\vert}
\newcommand{\proj}{\regtext{proj}}
\newcommand{\norm}[1]{\left\Vert#1\right\Vert}
\newcommand{\abs}[1]{\left\vert#1\right\vert}
\newcommand{\pow}[1]{\mathcal{P}\!\left(#1\right)}
\newcommand{\diag}[1]{\regtext{diag}\!\left(#1\right)}
\newcommand{\eig}[1]{\regtext{eig}\!\left(#1\right)}
\newcommand{\union}{\bigcup}
\newcommand{\intersection}{\bigcap}
\newcommand{\trans}{^\top}
\newcommand{\inv}{^{-1}}
\newcommand{\pinv}{^{\dagger}}
\newcommand{\sign}{\regtext{sign}}
\newcommand{\expm}{\regtext{exp}}
\newcommand{\logm}{\regtext{log}}
\newcommand{\skw}{_{\times}}
\newcommand{\bigO}{\mathcal{O}}
\newcommand{\bdry}[1]{\regtext{bd}\!\left(#1\right)}
\renewcommand{\ker}[1]{\regtext{ker}\!\left(#1\right)}
\newcommand{\conv}{\texttt{conv}}

\newcommand{\lbl}[1]{_{\regtext{#1}}}
\newcommand{\lo}{\lbl{lo}}
\newcommand{\hi}{\lbl{hi}}

\newcommand{\emptyarr}{[\ ]}
\newcommand{\zeros}{{0}}
\newcommand{\ones}{{1}}
\newcommand{\eye}{\regtext{I}}

\newcommand{\interval}[1]{[ #1 ]}
\newcommand{\iv}[1]{[ #1 ]}
\newcommand{\nom}[1]{#1}
\newcommand{\pz}[1]{\mathbf{#1}}
\newcommand{\pzgreek}[1]{\bm{#1}}
\newcommand{\PZ}[1]{\mathcal{PZ}\left(#1\right)}

\newcommand{\dist}{d}
\newcommand{\sdf}{s}
\newcommand{\rdf}{r}
\newcommand{\pdf}{\pi}
\newcommand{\ardf}{\Tilde{r}}
\newcommand{\asdf}{\Tilde{s}}
\newcommand{\rdfnn}{\Tilde{r}_{NN|\theta}}
\newcommand{\QP}{Alg. \ref{alg:rdf}\xspace}

\newcommand{\Aobs}{A_O}
\newcommand{\bobs}{b_O}
\newcommand{\hobs}{h_\regtext{obs}}

\newcommand{\obsset}{\mathscr{O}}


\newcommand{\pzk}[1]{\pz{ #1 ; k }}
\newcommand{\pzi}[1]{\pz{ #1 }(\pz{T_i};\pz{K})}
\newcommand{\pzki}[1]{\pz{ #1 }(\pz{T_i};k)}
\newcommand{\pzjki}[1]{\pz{ #1 }_j (\pz{T_i};k )}
\newcommand{\pzjKi}[1]{\pz{ #1 }_j (\pz{T_i};K )}

\newcommand{\ith}{$i$\ts{th}}
\newcommand{\jth}{$j$\ts{th}}

\newcommand{\pzqi}{\pzi{q}}
\newcommand{\pzqli}{\pzi{q_l}}
\newcommand{\pzqdi}{\pzi{\dot{q}}}
\newcommand{\pzqdai}{\pzi{\dot{q}_{a}}}
\newcommand{\pzqddi}{\pzi{\ddot{q}}}
\newcommand{\pzqddai}{\pzi{\ddot{q}_{a}}}
\newcommand{\pzqdesi}{\pzi{q_{d}}}
\newcommand{\pzqddesi}{\pzi{\dot{q}_{d}}}
\newcommand{\pzqdddesi}{\pzi{\ddot{q}_{d}}}
\newcommand{\pzqdeski}{\pzki{q_{d}}}
\newcommand{\pzqddeski}{\pzki{\dot{q}_{d}}}
\newcommand{\pzqdddeski}{\pzki{\ddot{q}_{d}}}
\newcommand{\pzui}{\pzki{u}}
\newcommand{\pzqki}{\pzki{q}}
\newcommand{\pzqdki}{\pzki{\dot{q}}}
\newcommand{\pzuki}{\pzki{u}}

\newcommand{\pzqji}{\pzi{q_j}}
\newcommand{\pzqdji}{\pzi{\dot{q}_j}}
\newcommand{\pzqdaji}{\pzi{\dot{q}_{a,j}}}
\newcommand{\pzqddji}{\pzi{\ddot{q}_j}}
\newcommand{\pzqddaji}{\pzi{\ddot{q}_{a,j}}}
\newcommand{\pzqdesji}{\pzi{q_{d,j}}}
\newcommand{\pzqddesji}{\pzi{\dot{q}_{d,j}}}
\newcommand{\pzqdddesji}{\pzi{\ddot{q}_{d,j}}}
\newcommand{\pzqdesjki}{\pzki{q_{d,j}}}
\newcommand{\pzqddesjki}{\pzki{\dot{q}_{d,j}}}
\newcommand{\pzqdddesjki}{\pzki{\ddot{q}_{d,j}}}
\newcommand{\pzuji}{\pzki{u_j}}
\newcommand{\pzqjki}{\pzki{q_j}}
\newcommand{\pzqdjki}{\pzki{\dot{q}_j}}
\newcommand{\pzujki}{\pzki{u_j}}

\newcommand{\pzujKi}{\pz{u}(\pzqAi, \nomparams, \intparams)}
\newcommand{\pzFKjki}{\pz{FK_j}(\pzqki)}
\newcommand{\pzFOjki}{\pz{FO_j}(\pzqki)}
\newcommand{\pzFOki}{\pz{FO}(\pzqki)}
\newcommand{\pzFOjkibuf}{\pz{FO_{j}^{buf}}(\pzqki)}
\newcommand{\pzFKjKi}{\pz{FK_j}(\pzqi)}
\newcommand{\pzFOjKi}{\pz{FO_j}(\pzqi)}
\newcommand{\Pj}{\mathcal{P}_j}
\newcommand{\Hjh}{\mathcal{H}_j^{(h)}}
\newcommand{\Ajh}{A_j^{(h)}}
\newcommand{\bjh}{b_j^{(h)}}

\newcommand{\pzg}{g}
\newcommand{\pzv}{x}
\newcommand{\pze}{\alpha}
\newcommand{\pzn}{{n_g}}
\newcommand{\pzgi}{g_i}
\newcommand{\pzei}{\alpha_i}

\newcommand{\tvar}{x_t}
\newcommand{\tvari}{x_{t_{i}}}

\newcommand{\q}{q(t)}
\newcommand{\qd}{\dot{q}(t)}
\newcommand{\qdd}{\ddot{q}(t)}
\newcommand{\qa}{q_a(t)}
\newcommand{\qadot}{\dot{q}_a(t)}
\newcommand{\qaddot}{\ddot{q}_a(t)}
\newcommand{\qak}{q_a(t; k)}
\newcommand{\qakdot}{\dot{q}_a(t; k)}
\newcommand{\qakddot}{\ddot{q}_a(t; k)}
\newcommand{\qdes}{q_d(t)}
\newcommand{\qdesdot}{\dot{q}_d(t)}
\newcommand{\qdesddot}{\ddot{q}_d(t)}
\newcommand{\qdesk}{q_d(t; k)}
\newcommand{\qdeskdot}{\dot{q}_d(t; k)}
\newcommand{\qdeskddot}{\ddot{q}_d(t; k)}


\newcommand{\qj}{q_j(t)}
\newcommand{\ql}{q_l(t)}
\newcommand{\qdj}{\dot{q}_{j}(t)}
\newcommand{\qddj}{\ddot{q}_{j}(t)}
\newcommand{\qkj}{q_j(t; k)}
\newcommand{\qdkj}{\dot{q}_{j}(t; k)}
\newcommand{\qddkj}{\ddot{q}_{j}(t; k)}
\newcommand{\qaj}{q_{a, j}(t)}
\newcommand{\qajdot}{\dot{q}_{a, j}(t)}
\newcommand{\qajddot}{\ddot{q}_{a, j}(t)}
\newcommand{\qdesj}{q_{d, j}(t)}
\newcommand{\qdesjdot}{\dot{q}_{d, j}(t)}
\newcommand{\qdesjddot}{\ddot{q}_{d, j}(t)}
\newcommand{\qdeskj}{q_{d, j}(t; k)}
\newcommand{\qdeskjdot}{\dot{q}_{d, j}(t; k)}
\newcommand{\qdeskjddot}{\ddot{q}_{d, j}(t; k)}

\newcommand{\nd}{n_d}
\newcommand{\nq}{n_q}
\newcommand{\nt}{n_t}
\newcommand{\nf}{n_f}
\newcommand{\nObs}{n_\mathscr{O}}
\newcommand{\nhj}{n_{h,j}}
\newcommand{\Nq}{ N_q }
\newcommand{\Nt}{ N_t }
\newcommand{\Nhj}{ N_{h,j} }
\newcommand{\NObs}{ N_{\mathscr{O}} }

\newcommand{\err}{e}
\newcommand{\errdot}{\dot{e}}
\newcommand{\errj}{e_j}
\newcommand{\errjdot}{\dot{e}_j}
\newcommand{\errddot}{\ddot{e}}
\newcommand{\robv}{v}
\newcommand{\robr}{r}
\newcommand{\robw}{w}
\newcommand{\robrdot}{\dot{r}}
\newcommand{\roblyap}{V(q, r)}
\newcommand{\roblyapmax}{\overline{V}(q, r)}
\newcommand{\roblyapdot}{\dot{V}(q, r))}
\newcommand{\robh}{h(q, r)}
\newcommand{\robH}{H}
\newcommand{\robhmin}{\underline{h}(q, r)}
\newcommand{\robhdot}{\dot{h}(q, r)}
\newcommand{\roblevel}{V_M}
\newcommand{\robcoeff}{\gamma}
\newcommand{\robKinf}{\alpha}
\newcommand{\robgain}{\alpha_c}
\newcommand{\ultbound}{\sqrt{\frac{2 \roblevel}{\sigma_m}}}
\newcommand{\pbound}{\epsilon_p}
\newcommand{\vbound}{\epsilon_v}
\newcommand{\pboundj}{\epsilon_{p, j}}
\newcommand{\vboundj}{\epsilon_{v, j}}
\newcommand{\pboundvec}{E_p}
\newcommand{\vboundvec}{E_v}
\newcommand{\epvar}{x_{e_p}}
\newcommand{\evvar}{x_{e_v}}
\newcommand{\epvarj}{x_{e_{p, j}}}
\newcommand{\evvarj}{x_{e_{v, j}}}

\providecommand{\R}{\ensuremath \mathbb{R}}
\newcommand{\plan}{_p}
\newcommand{\prev}{_\regtext{prev}}
\providecommand{\tfin}{t_\regtext{f}}

\newcommand{\zi}{z_i}
\newcommand{\zj}{z_j}
\newcommand{\rbf}{\mathbf{r}(t)}

\newcommand{\bM}{M}
\newcommand{\Mq}{M(\q, \Delta)}
\newcommand{\Mqdot}{\dot{M}(q, \Delta)}
\newcommand{\bMt}{\Tilde{M}}

\newcommand{\bC}{C}
\newcommand{\Cq}{C(\q, \qd)}
\newcommand{\Cqd}{C(\q, \qd, \Delta)}
\newcommand{\bCt}{\Tilde{C}}

\newcommand{\bG}{G}
\newcommand{\Gq}{G(\q)}
\newcommand{\Gqd}{G(\q, \Delta)}
\newcommand{\bGt}{\Tilde{G}}

\newcommand{\bfc}{\mathbf{c}}
\newcommand{\bfI}{\mathbf{I}}

\newcommand{\intparams}{[\Delta]}
\newcommand{\nomparams}{\Delta_0}
\newcommand{\trueparams}{\Delta}

\makeatletter
\newcommand{\smalloplus}{\mathbin{\mathpalette\make@small\oplus}}
\newcommand{\smallotimes}{\mathbin{\mathpalette\make@small\otimes}}

\newcommand{\Hquad}{\hspace{0.5em}} 

\newcommand{\RDF}{\regtext{\small{RDF}}}
\newcommand{\FK}{\regtext{\small{FK}}}
\newcommand{\IK}{\regtext{\small{IK}}}
\newcommand{\FO}{\regtext{\small{FO}}}
\newcommand{\IO}{\regtext{\small{IO}}}
\newcommand{\FS}{\regtext{\small{FS}}}
\newcommand{\IS}{\regtext{\small{IS}}}
\newcommand{\FC}{\regtext{\small{FC}}}
\newcommand{\IC}{\regtext{\small{IC}}}
\newcommand{\ID}{\regtext{\small{ID}}}

\newcommand{\exact}{^{\regtext{exact}}}
\newcommand{\slice}{\textnormal{\texttt{slice}}}
\newcommand{\eval}{\textnormal{\texttt{eval}}}
\newcommand{\stack}{\textnormal{\texttt{stack}}}
\newcommand{\reduce}{\textnormal{\texttt{reduce}}}
\newcommand{\getCoeffValue}{\texttt{getCoeffValue}}
\newcommand{\timeint}{([0, T])}

\newcommand{\SO}{\regtext{\small{SO}}}

\newcommand{\kj}{k_j}
\newcommand{\Kj}{K_j}

\newcommand{\kscale}{\eta_1}
\newcommand{\kjscale}{\eta_{j, 1}}
\newcommand{\koffset}{\eta_2}
\newcommand{\kjoffset}{\eta_{j, 2}}
\newcommand{\kvar}{x_k}
\newcommand{\kjvar}{x_{k_j}}

\newcommand{\F}{\mathcal{F}}

\newcommand{\eh}{\hat{e}}

\newcommand{\tsum}{{\textstyle\sum}}

\newcommand{\ujt}{u_j(t)}


\newcommand{\pjt}{p_j(t)}
\newcommand{\Rjt}{R_j(t)}

\newcommand{\unsafeobs}{_\regtext{obs}}
\newcommand{\unsafeself}{_\regtext{self}}
\newcommand{\unsafejoint}{_\regtext{lim}}
\newcommand{\jlim}{_\regtext{lim}}
\newcommand{\selfidx}{I\self}

\newcommand{\qlim}{q_{j,\regtext{lim}}}
\newcommand{\dqlim}{\dot{q}_{j,\regtext{lim}}}
\newcommand{\ddqlim}{\ddot{q}_{j,\regtext{lim}}}
\newcommand{\ulim}{u_{j,\regtext{lim}}}

\newcommand{\hitj}{h_i^{t, j}}
\newcommand{\buf}{_\regtext{buf}}
\newcommand{\slc}{_\regtext{slc}}
\newcommand{\Aitj}{A_i^{t, j}}
\newcommand{\bitj}{b_i^{t, j}}
\newcommand{\hitself}{h_{i_1, i_2}^{t}}
\newcommand{\Aitself}{A_{i_1, i_2}^{t}}
\newcommand{\bitself}{b_{i_1, i_2}^{t}}

\newcommand{\hqim}{h_{q_i^-}}
\newcommand{\hqip}{h_{q_i^+}}
\newcommand{\hdqim}{h_{\dot{q}_i^-}}
\newcommand{\hdqip}{h_{\dot{q}_i^+}}
\newcommand{\hijoint}{h_{i, \regtext{lim}}}

\newcommand{\initq}{q_{d_0}}
\newcommand{\initqj}{q_{d, j_{0}}}
\newcommand{\initdq}{\dot{q}_{d_0}}
\newcommand{\initdqj}{\dot{q}_{d, j_{0}}}
\newcommand{\initddq}{\ddot{q}_{d_0}}
\newcommand{\initddqj}{\ddot{q}_{d, j_{0}}}

\newcommand{\costfunc}{\phi}

\newcommand{\iss}{_{i}^{i}}
\newcommand{\issm}{_{i-1}^{i}}
\newcommand{\issmu}{_{i}^{i-1}}
\newcommand{\issmi}{_{i-1, i}^{i}}
\newcommand{\issp}{_{i+1}^{i}}
\newcommand{\issmm}{_{i-1}^{i-1}}
\newcommand{\isspp}{_{i+1}^{i+1}}
\newcommand{\issa}{_{a, i}^{i}}
\newcommand{\issc}{_{c, i}^{i}}
\newcommand{\issma}{_{a, i-1}^{i}}
\newcommand{\isspa}{_{a, i+1}^{i}}
\newcommand{\issmma}{_{a, i-1}^{i-1}}
\newcommand{\issppa}{_{a, i+1}^{i+1}}

\newcommand{\jss}{_{j}^{j}}
\newcommand{\jssm}{_{j-1}^{j}}
\newcommand{\jssmu}{_{j}^{j-1}}
\newcommand{\jssmj}{_{j-1, j}^{j}}
\newcommand{\jssp}{_{j+1}^{j}}
\newcommand{\jssmm}{_{j-1}^{j-1}}
\newcommand{\jsspp}{_{j+1}^{j+1}}
\newcommand{\jssa}{_{a, j}^{j}}
\newcommand{\jssc}{_{c, j}^{j}}
\newcommand{\jssma}{_{a, j-1}^{j}}
\newcommand{\jsspa}{_{a, j+1}^{j}}
\newcommand{\jssmma}{_{a, j-1}^{j-1}}
\newcommand{\jssppa}{_{a, j+1}^{j+1}}

\newcommand{\lssmu}{_{l}^{l-1}}

\newcommand{\qgoal}{q_\regtext{goal}}
\newcommand{\qstart}{q_\regtext{start}}
\newcommand{\methodname}{{RDF}\xspace}

\newcommand{\timestep}{\Delta t}

\maketitle
\thispagestyle{plain}
\pagestyle{plain}

\begin{abstract} 
Generating safe motion plans in real-time is a key requirement for deploying robot manipulators to assist humans in collaborative settings.
In particular, robots must satisfy strict safety requirements to avoid self-damage or harming nearby humans.
Satisfying these requirements is particularly challenging if the robot must also operate in real-time to adjust to changes in its environment.
This paper addresses these challenges by proposing Reachability-based Signed Distance Functions (\methodname{}s) as a neural implicit representation for robot safety.
\methodname{}, which can be constructed using supervised learning in a tractable fashion, accurately predicts the distance between the swept volume of a robot arm and an obstacle.
\methodname{}'s inference and gradient computations are fast and scale linearly with the dimension of the system; these features enable its use within a novel real-time trajectory planning framework as a continuous-time collision-avoidance constraint.
The planning method using \methodname{} is compared to a variety of state-of-the-art techniques and is demonstrated to successfully solve challenging motion planning tasks for high-dimensional systems faster and more reliably than all tested methods.
\end{abstract}

\section{Introduction}
\label{sec:intro}

Robotic manipulators will one day assist humans in a variety of collaborative tasks such as mining, farming, and surgery.
However, to ensure that they operate robustly in human-centric environments, they must satisfy several important criteria.
First, robots should be autonomous and capable of making their own decisions about how to accomplish specific tasks. 
Second, robots should be safe and only perform actions that are guaranteed to not damage objects in the environment, nearby humans, or even the robot itself.  
Third, robots should operate in real-time to quickly adjust their behavior as their environment or task changes.

Modern model-based motion planning frameworks are typically hierarchical and consist of three levels: a high-level planner, a mid-level trajectory planner, and a low-level tracking controller. 
The high-level planner generates a sequence of discrete waypoints between the start and goal locations of the robot. 
The mid-level trajectory planner computes time-dependent velocities and accelerations at discrete time instances that move the robot between consecutive waypoints. 
The low-level tracking controller attempts to track the trajectory as closely as possible. 
While variations of this motion planning framework have been shown to work on multiple robotic platforms \cite{holmes2020armtd}, there are still several limitations that prevent wide-scale, real-world adoption of this method. 
For instance, this approach can be computationally expensive especially as the robot complexity increases, which can make it impractical for real-time applications. 
By introducing heuristics such as reducing the number of discrete time instances where velocities and accelerations are computed at the mid-level planner, many algorithms achieve real-time performance at the expense of robot safety. 
Unfortunately, this increases the potential for the robot to collide with obstacles.

\begin{figure}[t]
    \centering
    \includegraphics[width=0.93\columnwidth]{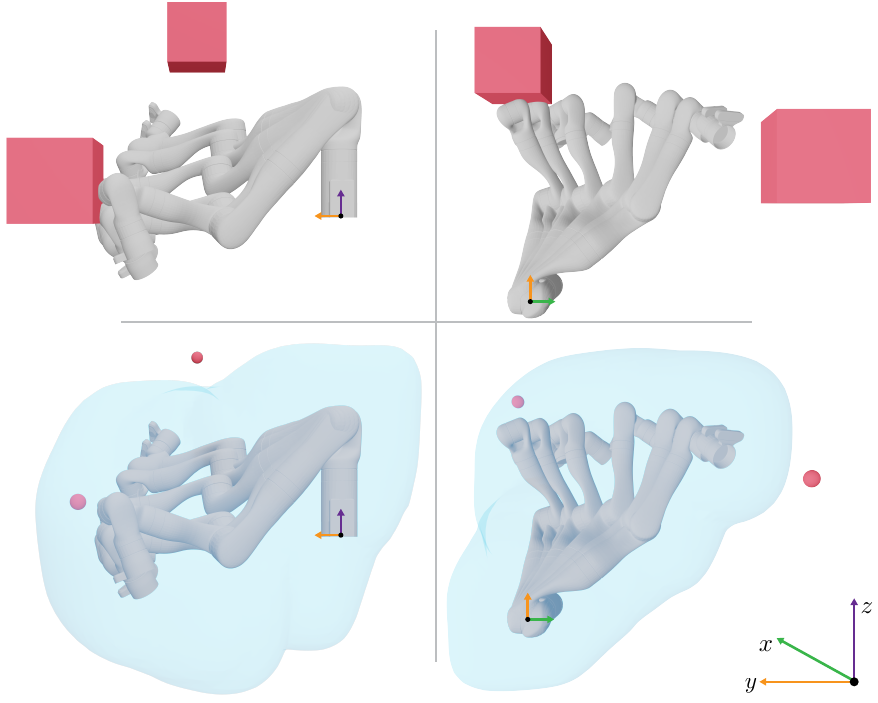}
    \caption{\methodname{} is a neural implicit safety representation that computes distances between the swept volume of a robotic arm and obstacles within a receding-horizon trajectory planning framework. The top panels show two different orthographic views of a single planning iteration with several intermediate poses of the robot arm in collision with one of the obstacles (red cubes). The bottom panels show orthographic views of the 3D  reconstruction of \methodname{}s signed distance field (transparent blue) with one of the obstacle centers (red spheres) interior to \methodname{}'s zero-level set.}
    \label{fig:network}
    \vspace*{-0.5cm}
\end{figure}

To resolve this challenge, this paper proposes Reachability-based Signed Distance Functions (\methodname{}s), a neural implicit representation for robot safety that can be effectively embedded within a mid-level trajectory optimization algorithm.
\methodname{} is a novel variation of the signed distance function (SDF) \cite{park_deepsdf_2019,gropp2020implicit,deep_lipschitz_sdf} that computes the distance between the swept volume (\emph{i.e.}, reachable set) of robot manipulator and obstacles in its environment.
We use \methodname{}s within a receding-horizon trajectory planning framework to enforce safety by implicitly encoding obstacle-avoidance constraints.
\methodname{} has advantages over traditional model-based representations for obstacle-avoidance.
First, by approximating the swept volume \methodname{} learns a continuous-time representation for robot safety.
Second, \methodname{} replaces the need to do computationally expensive collision checking at every planning iteration with a rapidly computable forward pass of the network.
Third, \methodname{}'s fast inference and gradient computations make it ideal as a constraint in trajectory optimization problems.
Fourth, as we illustrate in this paper, \methodname{} scales better than existing planning algorithms with respect to the dimension of the system.

\subsection{Related Work}
\label{sec:related_work}
Our approach lies at the intersection of swept volume approximation, neural implicit shape representation, and robot motion planning.
We summarize algorithms in these areas and highlight their computational tractability. 

Swept volume computation \cite{blackmore_differential_1990, blackmore_analysis_1992, blackmore_analysis_1994, blackmore_sweep-envelope_1997, blackmore_trimming_1999} has a rich history in robotics \cite{abdel-malek_swept_2006} where it has been used for collision detection during motion planning \cite{spatial_planningPerez}.
Because computing exact swept volumes of articulated robots is analytically intractable \cite{abdel-malek_swept_2006}, many algorithms rely on approximations using convex polyhedra, occupancy grids, and CAD models \cite{Campen2010PolygonalBE, Kim2003FastSV, Gaschler2013RobotTA}.
However, these methods often suffer from high computational costs and are generally not suitable when generating complex robot motions and as a result, are only applied while performing offline motion planning \cite{perrin2012svfootstep}.
To address some of these limitations, a recent algorithm was proposed to compute a single probabilistic road map offline whose safety was verified online by performing a parallelized collision check with a precomputed swept volume at run-time \cite{Murray2016RobotMP}.
However, this method was not used for real-time trajectory generation. 

An alternative to computing swept volumes is to buffer the robot and perform collision-checking at discrete time instances along a given trajectory.
This is common with state-of-art trajectory optimization approaches such as CHOMP \cite{zucker2013chomp} and TrajOpt \cite{schulman2014trajopt}.
Although these methods have been demonstrated to generate robot motion plans in real-time by reducing the number of discrete time instances where collision checking occurs, they cannot be considered safe as they enforce collision avoidance via soft penalties in the cost function.

A recent approach called Reachability-based Trajectory Design (RTD) \cite{kousik2019_quad_RTD} combines both swept volume approximation and trajectory optimization to generate safe motion plans in real-time.
At runtime, RTD uses zonotope arithmetic to build a differentiable reachable set representation that overapproximates the swept volume corresponding to a continuum of possible trajectories the robot could follow.
It then solves an optimization problem to select a feasible trajectory such that the subset of the swept volume encompassing the robot's motion is not in collision.
Importantly, the reachable sets are constructed so that obstacle avoidance is satisfied in continuous-time.
While extensions of RTD have demonstrated real-time, certifiably-safe motion planning for robotic arms \cite{holmes2020armtd, armour2023arxiv} with seven degrees of freedom, applying RTD to higher dimensional systems is challenging because, as we illustrate in this paper, it is unable to construct reach sets rapidly.

A growing trend in machine learning and computer vision is to implicitly represent 3D shapes with learned neural networks.
One seminal work in this area is DeepSDF \cite{park_deepsdf_2019}, which was the first approach to learn a continuous volumetric field that reconstructs the zero-level set of entire classes of 3D objects.
Gropp et. al \cite{gropp2020implicit} improved the training of SDFs by introducing an \emph{Eikonal} term into their loss function that acts as an implicit regularizer to encourage the network to have unit $l2$-norm gradients.
Neural implicit representations have also been applied to robotics problems including motion planning \cite{adamkiewicz2022vision, huhneural}, mapping \cite{ortiz_isdf_2022, camps2022learning}, and manipulation \cite{ichnowski2021dex, driess2022learning, yen2022nerf}.
Particularly relevant to our current approach is the work by Koptev et. al \cite{koptev2022implicit}, which learned an SDF as an obstacle-avoidance constraint for safe reactive motion planning.
Similar to approaches described above, \cite{koptev2022implicit} only enforces safety constraints at discrete time points.

\subsection{Contributions}
The present work investigates learned neural implicit representations with reachability-based trajectory design for fast trajectory planning. 
The contributions of this paper are as follows:
\begin{enumerate}[label={\arabic*.}]
    \item A neural implicit representation called \methodname{} that computes the distance between a parameterized swept volume trajectory and obstacles;
    \item An efficient algorithm to generate training data to construct \methodname{}; and
    \item A novel, real-time optimization framework that utilizes \methodname{} to construct a collision avoidance constraint.
\end{enumerate}
To illustrate the utility of this proposed optimization framework we apply it to perform real-time motion planning on a variety of robotic manipulator systems and compare it to several state of the art algorithms. 

The remainder of the paper is organized as follows:
Section \ref{sec:preliminaries} summarizes the set representations used throughout the paper;
Section \ref{sec:arm} describes the safe motion planning problem of interest;
Section \ref{sec:RDF} defines several distance functions that are used to formulate RDF;
Section \ref{sec:proposed_method} formulates the safe motion planning problem;
Section \ref{sec:NN} describes how to build \methodname{} and use it within a safe motion planning framework;
Sections \ref{sec:experimental_setup} and \ref{sec:results} summarize the evaluation of the proposed method on a variety of different example problems.

\section{Preliminaries}
\label{sec:preliminaries}

This section describes the representations of sets and operations on these representations that we use throughout the paper.
Note all norms that are not explicitly specified are the $2$-norm. 

\subsection{Zonotopes and Polynomial Zonotopes}

We begin by defining zonotopes, matrix zonotopes, and polynomial zonotopes.

\begin{defn}\label{defn:zonotope}
A \emph{zonotope} $Z \subset \R^n$ is a convex, centrally-symmetric polytope defined by a \emph{center} $c \in \R^n$, \emph{generator matrix} $G \in \R^{n\times\pzn}$, and \emph{indeterminant vector} $\beta \in \R^\pzn$:
\begin{equation}
\label{eq:zonotope}
    Z = \left\{ z \in \R^n \, \mid \, 
            z = c + G\,\beta, \, ||\beta||_{\infty}\leq1
        \right\}
\end{equation}
where there are $\pzn \in \N$ generators.
When we want to emphasize the center and generators of a zonotope, we write $Z = (c,G)$.
\end{defn}

\begin{defn}
A \emph{matrix zonotope} $Z \subset \R^n$ is a convex, centrally-symmetric polytope defined by a \emph{center} $c \in \R^n$, \emph{generator matrix} $G \in \R^{n\times\pzn}$, and \emph{indeterminant vector} $\beta \in \R^\pzn$:
\begin{equation}
\label{eq:matrix_zonotope}
    Z = \left\{ z \in \R^n \, \mid \, 
            z = c + G\,\beta, \, ||\beta||_{\infty}\leq1
        \right\}
\end{equation}
where there are $\pzn \in \N$ generators.
\end{defn}


\begin{defn}[Polynomial Zonotope]
\label{def:pz}
    A \defemph{polynomial zonotope} $\pz{P} \subset \R^n$ is given by its generators $\pzgi \in \R^n$ (of which there are $\pzn$), exponents $\pzei \in \N^\pzn$, and indeterminates $\pzv \in [-1,1]^\pzn$ as
    \begin{align}\label{eq:pz_definition}
        \pz{P} = \PZ{\pzgi, \pzei, \pzv} = \left\{
                z \in \R^n \, \mid \,
                z = \sum_{i=0}^{\pzn} \pzgi \pzv ^{\pzei}, \, \pzv \in [-1, 1]^\pzn
            \right\}.
    \end{align}
    We refer to $\pzv ^{\pzei}$ as a \defemph{monomial}.
    A \defemph{term} $\pzgi\pzv ^{\pzei}$ is produced by multiplying a monomial by the associated generator $\pzgi$.
\end{defn}
Note that one can think of a zonotope as a special case of polynomial zonotope where one has an exponent made up of all zeros and the remainder of exponents only have one non-zero element that is equal to one.
As a result, whenever we describe operations on polynomial zonotopes they can be extended to zonotopes.
When we need to emphasize the generators and exponents of a polynomial zonotope, we write $\pz{P} = \PZ{\pzgi, \pzei, \pzv}$.
Throughout this document, we exclusively use bold symbols to denote polynomial zonotopes.

\subsection{Operations on Zonotopes and Polynomial Zonotopes}
This section describes various set operations.

\subsubsection{Set-based Arithmetic}

Given a set $\Omega \subset \mathbb{R}^{\nd}$, let $\partial \Omega \subset \mathbb{R}^{\nd}$ be its boundary and $\Omega^{c} \subset \mathbb{R}^{\nd}$ denote its complement.

\begin{defn}\label{def:convhull}
The \defemph{convex hull operator} $\conv: \mathbb{R}^{\nd} \to \mathbb{R}^{\nd} $ is defined by
\begin{equation}
    \conv(\Omega) = \bigcap C_{\alpha}
\end{equation}
where $C_{\alpha}$ is a convex set containing $\Omega$.
\end{defn}

Let $U$, $V \subset \R^n$.
The \defemph{Minkowski Sum} is $U \oplus V  = \{ u + v \, \mid \, u \in U, v \in V \}$; 
the \defemph{Multiplication} of $UV = \{ u v \, \mid \, u \in U, v \in V \}$ where all elements in $U$ and $V$ must be appropriately sized to ensure that their product is well-defined.

\subsubsection{Polynomial Zonotope Operations}

As described in Tab. \ref{tab:poly_zono_operations}, there are a variety of operations that we can perform on polynomial zonotopes (\emph{e.g.}, minkowski sum, multiplication, etc.).
The result of applying these operations is a polynomial zonotope that either exactly represents or over approximates the application of the operation on each individual element of the polynomial zonotope inputs.
The operations given in Tab. \ref{tab:poly_zono_operations} are rigorously defined in \cite{armour2023arxiv}.
A thorough introduction to polynomial zonotopes can be found in \cite{kochdumper2020sparse}.

One desirable property of polynomial zonotopes is the ability to obtain subsets by plugging in values of known indeterminates.
For example, say a polynomial zonotope $\pz{P}$ represented a set of possible positions of a robot arm operating near an obstacle.
It may be beneficial to know whether a particular choice of $\pz{P}$'s indeterminates yields a subset of positions that could collide with the obstacle.
To this end, we introduce the operation of ``slicing'' a polynomial zonotope $\pz{P} = \PZ{ \pzgi, \pzei, \pzv }$  by evaluating an element of the indeterminate $\pzv$.
Given the $j$\ts{th} indeterminate $\pzv_j$ and a value $\sigma \in [-1, 1]$, slicing yields a subset of $\pz{P}$ by plugging $\sigma$ into the specified element $\pzv_j$:
\begin{equation}
    \label{eq:pz_slice}
   \hspace*{-0.25cm} \setop{slice}(\pz{P}, \pzv_j, \sigma) \subset \pz{P} =
        \left\{
            z \in \pz{P} \, \mid \, z = \sum_{i=0}^{\pzn} \pzgi \pzv ^{\pzei}, \, \pzv_j = \sigma
        \right\}.
\end{equation}

One particularly important operation that we require later in the document, is an operation to bound the elements of a polynomial zonotope.
It is possible to efficiently generate these upper and lower bounds on the values of a polynomial zonotope through overapproximation.
In particular, we define the $\setop{sup}$ and $\setop{inf}$ operations which return these upper and lower bounds, respectively, by taking the absolute values of generators.
For $\pz{P} \subseteq \R^n$, these return
\begin{align}
    \setop{sup}(\pz{P}) = g_0 + \sum_{i=1}^{\pzn} \abs{\pzgi}, \label{eq:pz_sup}\\
    \setop{inf}(\pz{P}) = g_0 - \sum_{i=1}^{\pzn} \abs{\pzgi}. \label{eq:pz_inf}
\end{align}
Note that for any $z \in \pz{P}$,  $\setop{sup}(\pz{P}) \geq z$ and $\setop{inf}(\pz{P}) \leq z$, where the inequalities are taken element-wise.
These bounds may not be tight because possible dependencies between indeterminates are not accounted for, but they are quick to compute.

\begin{table}[t]
    \centering
    \begin{tabular}{c|c}
        Operation & Computation \\
        \hline
        $\pz{P}_1 \oplus \pz{P}_2$ (PZ Minkowski Sum) (\cite{armour2023arxiv}, eq. (19)) & Exact \\
        $\pz{P}_1\pz{P}_2$ (PZ Multiplication) (\cite{armour2023arxiv}, eq. (21)) & Exact \\
        $\setop{slice}(\pz{P}, \pzv_j, \sigma)$ (\cite{armour2023arxiv}, eq. (23)) & Exact \\
        $\setop{inf}(\pz{P})$ (\cite{armour2023arxiv}, eq. (24))  & Overapproximative \\
        $\setop{sup}(\pz{P})$ (\cite{armour2023arxiv}, eq. (25))   & Overapproximative \\
        $f(\pz{P}_1) \subseteq \pz{P}_2$ (Taylor expansion) (\cite{armour2023arxiv}, eq. (32)) & Overapproximative 
    \end{tabular}
    \caption{Summary of polynomial zonotope operations.
    }
    \label{tab:poly_zono_operations}
\end{table}


\section{Arm and Environment}
\label{sec:arm}

This section summarizes the robot and environmental model that is used throughout the remainder of the paper. 

\subsection{Robotic Manipulator Model}
Given an $\nq$ degree of freedom serial robotic manipulator with configuration space $Q$ and a compact time interval $T \subset \R$ we define a trajectory for the configuration as $q: T \to Q \subset \R^{\nq}$. 
The velocity of the robot is $\dot{q}: T \to \R^{\nq}$.
Let $\Nq = \{1,\ldots,\nq\}$.
We make the following assumptions about the structure of the robot model:
\begin{assum}
The robot operates in an $\nd$-dimensional workspace, which we denote $W \subset \R^{\nd}$.
The robot is composed of only revolute joints, where the $j$\ts{th} joint actuates the robot's $j$\ts{th} link.
The robot's $j$\ts{th} joint has position and velocity limits given by $\qj \in [\qlim^-, \qlim^+]$ and $\qdj \in [\dqlim^-, \dqlim^+]$ for all $t \in T$, respectively.
Finally, the robot is fully actuated, where the robot's input is given by $u: T \to \R^{\nq}$. 
\end{assum}
\noindent One can make the one-joint-per-link assumption without loss of generality by treating joints with multiple degrees of freedom (\emph{e.g.}, spherical joints) as links with zero length.
Note that we use the revolute joint portion of this assumption to simplify the description of forward kinematics; however, these assumptions can be easily extended to more complex joint using the aforementioned argument or can be extended to prismatic joints in a straightforward fashion.

Note that the lack of input constraints means that one could apply an inverse dynamics controller \cite{spong2005textbook}  to track any trajectory of of the robot perfectly. 
As a result, we focus on modeling the kinematic behavior of the manipulator. 
Note, that the approach presented in this paper could also be extended to deal with input limits using a dynamic model of the manipulator; however, in the interest of simplicity we leave that extension for future work.

\subsubsection{Arm Kinematics}
Next, we introduce the robot's kinematics. 
Suppose there exists a fixed inertial reference frame, which we call the \textit{world} frame.
In addition suppose there exists a \textit{base} frame, which we denote the $0$\ts{th} frame, that indicates the origin of the robot's kinematic chain.
We assume that the $j$\ts{th} reference frame $\{\hat{x}_j, \hat{y}_j, \hat{z}_j\}$ is attached to the robot's $j$\ts{th} revolute joint, and that $\hat{z}_j = [0, 0, 1]^\top$ corresponds to the $j$\ts{th} joint's axis of rotation.
Then for a configuration at a particular time, $\q$, the position and orientation of frame $j$ with respect to frame $j-1$ can be expressed using homogeneous transformations \cite[Ch. 2]{spong2005textbook}:

\begin{equation}
\label{eq:homogeneous_transform}
    \homtrans_{j}^{j-1}(\qj) = 
    \begin{bmatrix} R\jssmu (\qj) & p\jssmu \\
    \zeros  & 1 \\
    \end{bmatrix},
\end{equation}
where $R\jssmu (\qj)$ is a configuration-dependent rotation matrix and $p\jssmu$ is the fixed translation vector from frame $j-1$ to frame $j$.

With these definitions, we can express the forward kinematics of the robot.
Let $\FK_j: Q \to \R^{4 \times 4}$ map the robot's configuration to the position and orientation of the $j$\ts{th} joint in the world frame:
\begin{equation}\label{eq:fk_j}
    \FK_j(\q) = 
    \prod_{l=1}^{j} \homtrans_{l}^{l-1}(q_l(t)) = 
    \begin{bmatrix} R_j(\q) & p_j(\q) \\
    \zeros  & 1 \\
    \end{bmatrix},
\end{equation}
where 
\begin{equation}
    R_j(\q) \coloneqq R_j^{0}(\q) = \prod_{l=1}^{j}R_{l}^{l-1}(\ql)
\end{equation}
and 
\begin{equation}
    p_j(\q) = \sum_{l=1}^{j} R_{l}(\q) p_{l}^{l-1}.
\end{equation}

\subsection{Arm Occupancy}
Next, we define the forward occupancy of the robot by using the arm's kinematics to describe the volume occupied by the arm in the workspace.
Let $\pz{L_j} \subset \R^3$ denote a polynomial zonotop overapproximation to the volume occupied by the $j$\ts{th} link with respect to the $j$\ts{th} reference frame. 
The forward occupancy of link $j$ is the map $\FO_j: Q \to \pow{W}$ defined as
\begin{align}\label{eq:forward_occupancy_j}
     \FO_j(\q) &= p_j(\q) \oplus R_j(\q) L_j,
\end{align}
where the first expression gives the position of joint $j$ and the second gives the rotated volume of link $j$.
The volume occupied by the entire arm in the workspace is given by the function $\FO: Q \to \pow{W}$ that is defined as
\begin{align}\label{eq:forward_occupancy}
    \FO(\q) = \bigcup_{j = 1}^{\nq} \FO_j(\q) \subset W. 
\end{align}
For convenience, we use the notation $\FO(q(T))$ to denote the forward occupancy over an entire interval $T$.

\subsection{Environment}
Next, we describe the arm's environment and its obstacles.

\subsubsection{Obstacles}
The arm must avoid obstacles in the environment while performing motion planning. 
These obstacles satisfy the following assumption:
\begin{assum}
\label{assum:obstacles}
The transformation between the world frame of the workspace and the base frame of the robot is known, and obstacles are represented in the base frame of the robot.
At any time, the number of obstacles $\nObs \in \N$ in the scene is finite.
Let $\obsset$ be the set of all obstacles $\{ O_1, O_2, \ldots, O_{\nObs} \}$.
Each obstacle is convex, bounded, and static with respect to time.
The arm has access to a zonotope that overapproximates the obstacle's volume in workspace.
Each zonotope overapproximation of the obstacle has the same volume and is an axis-aligned cube.
\end{assum}
\noindent A convex, bounded object can always be overapproximated as a zonotope \cite{guibas2003zonotopes}.
In addition, if one is given a non-convex bounded obstacle, then one can outerapproximate that obstacle by computing its convex hull. 
If one has an obstacle that is larger than the pre-fixed axis-aligned cube, then one can introduce several axis-aligned cubes whose union is an overapproximation to the obstacle. 
Note because we use the zonotope overapproximation during motion planning, we conflate the obstacle with its zonotope overapproximation throughout the remainder of this document. 

Dynamic obstacles may also be considered within the \methodname framework by introducing a more general notion of safety \cite[Definition 11]{vaskov2019towards}, but we omit this case in this paper to ease exposition.
Finally, if a portion of the scene is occluded then one can treat that portion of the scene as an obstacle.
 We say that the arm is \textit{in collision} with an obstacle if $\FO_j(\q) \cap O_{\ell} \neq \emptyset$ for any $j \in \Nq$ or $\ell \in \NObs$ where $\NObs = \{1,\ldots,\nObs\}$. 

\subsection{Trajectory Design}
\label{subsec:trajectory}
Our goal is to develop an algorithm to compute safe trajectories in a receding-horizon manner by solving an optimization program over parameterized trajectories at each planning iteration.
These parameterized trajectories are chosen from a pre-specified continuum of trajectories, with each uniquely determined by a \textit{trajectory parameter} $k \in K \subset \R^{n_k}$, $n_k \in \N$.
$K$ is compact and can be designed in a task dependent or robot morphology-specific fashion \cite{holmes2020armtd, kousik2020bridging, kousik2019_quad_RTD, liu2022refine}, as long as it satisfies the following properties.

\begin{defn}[Trajectory Parameters]
\label{defn:traj_param}
For each $k \in K$, a \emph{parameterized trajectory} is an analytic function $q(\,\cdot\,; k) : T \to Q$ that satisfies the following properties:
\begin{outline}[enumerate]
\1 The parameterized trajectory starts at a specified initial condition $(q_0, \dot{q}_0)$, so that $q(0; k) = q_0$, and $\dot{q}(0; k) = \dot{q}_0$.
\1 $\dot{q}(\tfin; k) = 0$  (i.e., each parameterized trajectory brakes to a stop, and at the final time has zero velocity).
\end{outline}
\end{defn}
\noindent The first property allows for parameterized trajectories to be generated online.
In particular, recall that \methodname performs real-time receding horizon planning by executing a desired trajectory computed at a previous planning iteration while constructing a desired trajectory for the subsequent time interval.
The first property allows parameterized trajectories that are generated by \methodname to begin from the appropriate future initial condition of the robot.
The second property ensures that a fail safe braking maneuver is always available.


\section{Reachability-based Signed Distance Functions}
\label{sec:RDF}

This section defines the signed distance function between sets. 
Signed distance functions are used in robotics in a variety of applications including representing collision avoidance constraints. 
This section describes how to extend the signed distance function to a distance function between the forward occupancy of a robot and an obstacle. 
This novel distance function, which we call the reachability-based signed distance function (\methodname{}), enables us to formulate the collision avoidance problem between a parameterized reachable set and an obstacle as an optimization problem. 


\subsection{Overview of Signed Distance Fields}

We begin by defining an unsigned distance function:
\begin{defn}
Given a set $\Omega \subset \mathbb{R}^{\nd}$, the \defemph{distance function} associated with $\Omega$ is defined by
    \begin{equation}
       \dist(x;\Omega) = \min_{y \in \partial\Omega} \|x - y\|.
    \end{equation}
The \defemph{distance between two sets} $\Omega_1, \Omega_2 \subset \mathbb{R}^{\nd}$ is defined by
    \begin{equation}
       \dist(\Omega_1,\Omega_2) = \min_{\substack{x \in \partial\Omega_1 \\ y \in \partial\Omega_2}} \|x - y\|.
    \end{equation}
\end{defn}
Notice that this distance function is zero for sets that have non-trivial intersection. 
As a result, this distance function provides limited information for such sets (i.e., it is unclear how much they are intersecting with one another).
To address this limitation, we consider the following definition:
\begin{defn} Given a subset $\Omega$ of $\mathbb{R}^{\nd}$, the \defemph{signed distance function} $\sdf$ between a point is a map $\sdf: \mathbb{R}^{\nd} \to \mathbb{R}$ defined as
    \begin{equation}
    \sdf(x; \Omega) = 
        \begin{cases}
          \dist(x, \partial\Omega)  & \text{if } x \in \Omega^{c} \\
          -\dist(x,\partial\Omega) & \text{if } x \in \Omega.
        \end{cases}
\end{equation}
The \defemph{signed distance between two sets} $\Omega_1, \Omega_2 \subset \mathbb{R}^{\nd}$ is defined as
    \begin{equation}
    \sdf(\Omega_1,\Omega_2) = 
        \begin{cases}
          \dist(\Omega_1,\Omega_2)  & \text{if } \Omega_1 \cap \Omega_2 = \emptyset  \\
          -\dist(\Omega_1,\Omega_2) & \text{otherwise}.
        \end{cases}
\end{equation}
\end{defn}
Note that signed distance functions are continuous \cite{dapogny2012computation}, differentiable almost everywhere \cite{dapogny2012computation,evans2022partial}, and satisfy the \emph{Eikonal} equation:
\begin{defn}
Suppose $\sdf$ is the signed distance function associated with a set $\Omega \subset R^{\nd}$. 
Then the gradient of $\sdf$ satisifes the Eikonal Equation which is defined as
    \begin{equation}
   \| \nabla \sdf(x) \| = 1.
\end{equation}
\end{defn}
\noindent We use this property to construct our loss term in \ref{subsec:loss}

\subsection{Reachability-Based Signed Distance Functions}
This subsection describes the reachability-based distance function as the signed distance function associated with forward occupancy of a robot.
\begin{defn}\label{def:rdf_point}
The \defemph{reachability-based distance} function associated with the forward occupancy reachable set $\FO(q(T;k))$ is a mapping 
defined by
\begin{equation}
    \rdf(x, \FO_j(q(T;k))) = \min_{j \in \Nq} \rdf_j(x;\FO_j(q(T;k)))
\end{equation}
where
$\rdf_j$ is the signed distance function associated with the $j^{th}$ forward occupancy $\FO_j$ such that
\begin{equation}
    \rdf_j(x;\FO_j(q(T;k))) = \sdf(x;\FO_j(q(T;k))).
\end{equation}
The \defemph{reachability-based distance} between an obstacle $O \subset \mathbb{R}^d$ and the robot's forward occupancy reachable set $\FO$ is defined by
\begin{equation}
    \rdf(O, \FO(q(T;k))) = \min_{j \in \Nq} \sdf(O, \FO_j(q(T;k))).
\end{equation}
\end{defn}
One can use this distance function to formulate trajectory optimization problems as we describe next. 


\section{Formulating the Motion Planning Problem Using Polynomial Zonotopes}\label{sec:proposed_method}



To construct a collision free trajectory in a receding-horizon fashion, one could try to solve the following nonlinear optimization problem at each planning iteration:
\begin{align}
    \label{eq:optcost}
    &\underset{k \in K}{\min} &&\texttt{cost}(k) \\
    &&& \qkj \subseteq [\qlim^-, \qlim^+] \qquad &\forall j\in\Nq,t\in T\\
    &&& \qdkj \subseteq [\dqlim^-, \dqlim^+] \qquad &\forall j\in\Nq,t\in T \\
    &&& \label{eq:rdf-col-avoid}
    \rdf(O_{\ell}, \FO(q(T;k)))> 0 \qquad &\forall \ell\in\NObs 
\end{align}
The cost function \eqref{eq:optcost} specifies a user-defined objective, such as bringing the robot close to some desired goal.
Each of the constraints guarantee the safety of any feasible trajectory parameter.
The first two constraints ensure that the trajectory does not violate the robot's joint position and velocity limits.
The last constraint ensures that the robot does not collide with any obstacles in the environment.
Note in this optimization problem, we have assumed that the robot does not have to deal with self-intersection constraints. 

Implementing a real-time optimization algorithm to solve this problem is challenging for several reasons.
First, the constraints associated with obstacle avoidance are non-convex. 
Second, the constraints must be satisfied for all time $t$ in an uncountable set $T$.
To address these challenges, a recent paper proposed to represent the trajectory and the forward occupancy of the robot using a polynomial zonotope representation \cite{armour2023arxiv}. 
We summarize these results below.

\subsection{Time Horizon and Trajectory Parameter PZs}
We first describe how to create polynomial zonotopes representing the planning time horizon $T$.
We choose a timestep $\timestep$ so that $\nt \coloneqq \frac{T}{\timestep} \in \N$.
Let $N_t := \{1,\ldots,\nt\}$.
Divide the compact time horizon $T \subset \R$ into $\nt$ time subintervals.
Consider the $i$\ts{th} time subinterval corresponding to $t \in \iv{(i-1)\timestep, i\timestep}$.
We represent this subinterval as a polynomial zonotope $\pz{T_i}$, where 
\begin{equation}
    \label{eq:time_pz}
    \pz{T_i} =
        \left\{t \in T \mid 
            t = \tfrac{(i-1) + i}{2}\timestep + \tfrac{1}{2} \timestep \tvari,\ \tvari \in [-1,1]
        \right\}
\end{equation}
with indeterminate $\tvari \in \iv{-1, 1}$.

Now we describe how to create polynomial zonotopes representing the set of trajectory parameters $K$.
In this work, we choose $K = \bigtimes_{i=1}^{n_q} K_i$, where each $\Kj$ is the compact one-dimensional interval  $\Kj = \iv{-1, 1}$.
We represent the interval $\Kj$ as a polynomial zonotope $\pz{\Kj} = \kjvar$ where $\kjvar \in \iv{-1, 1}$ is an indeterminate.

\subsection{Parameterized Trajectory and Forward Occupancy PZs}
The parameterized position and velocity trajectories of the robot, defined in Def. \ref{defn:traj_param}, are functions of both time $t$ and the trajectory parameter $k$.
Using the time partition and trajectory parameter polynomial zonotopes described above, we create polynomial zonotopes $\pzqji$ that overapproximate $\qkj$ for all $t$ in the $i$\ts{th} time subinterval and $k \in K$ by plugging the polynomial zonotopes $\pz{T_i}$ and $\pz{K}$ into the formula for $\qkj$.

Recall that $\pz{T_i}$ and $\pz{K_j}$ have indeterminates $\tvari$ and $\kjvar$, respectively.
Because the desired trajectories only depend on $t$ and $k$, the polynomial zonotopes $\pzqji$ and $\pzqdji$  depend only on the indeterminates $\tvari$ and $\kvar$.
By plugging in a given $k$ for $\kvar$ via the $\setop{slice}$ operation, we obtain a polynomial zonotope where $\tvari$ is the only remaining indeterminate. 
Because we perform this particular slicing operation repeatedly throughout this document, if we are given a polynomial zonotope, $\pzqdesji$, we use the shorthand $\pzqjki = \setop{slice}(\pzqji, \kvar, k)$.
Importantly, one can apply \cite[Lemma 17]{armour2023arxiv} to prove that the sliced representation is over approximative as we restate below:
\begin{lem}[Parmaeterized Trajectory PZs]
\label{lem:pz_desired_trajectory}
The parameterized trajectory polynomial zonotopes $\pzqdesji$ are overapproximative, i.e., for each $j \in \Nq$ and $k\in \pz{K},$ 
\begin{equation}
    \qkj \in \pzqjki \quad \forall t \in \pz{T_i}
\end{equation}
One can similarly define $\pzqdji$ that are also overapproximative.
\end{lem}

Next, we describe how to use this lemma to construct an overapproximative representation to the forward occupancy. 
In particular, because the rotation matrices $R\jssmu(q_j(t;k))$ depend on $\cos{(q_j(t;k))}$ and $\sin{(q_j(t;k))}$ one can compute  $\cos{(\pzqji)}$ and $\sin{(\pzqji)}$ using Taylor expansions as in  (\cite{armour2023arxiv}, eq. (32)).
By using this property and the fact that all operations involving polynomial zonotopes are either exact or overapproximative, the polynomial zonotope forward occupancy can be computed and proven to be overapproximative:
\begin{lem}[PZ Forward Occupancy]
\label{lem:pz_FO}
Let the polynomial zonotope forward occupancy reachable set for the $j$\ts{th} link at the $i$\ts{th} time step be defined as:
\begin{align}\label{eq:pz_forward_occupancy_j}
     \pzFOjKi = \pz{p_j}(\pzqi) \oplus \pz{R_j}(\pzqi)\pz{L_j},
\end{align}
then for each $j \in \Nq$,  $k \in \pz{K}$, $\FO_j(q(t;k)) \in  \pzFOjki$ for all $t \in \pz{T_i}$.
\end{lem}
\noindent 
For convenience, let
\begin{equation}
    \label{eq:pz_forward_occupancy}
    \pz{FO}(\pzqi) = \bigcup_{j = 1}^{\nq} \pzFOjKi.
\end{equation}

\subsection{PZ-based Optimization Problem}

Rather than solve the optimization problem described in \eqref{eq:optcost} -- \eqref{eq:rdf-col-avoid}, \cite{armour2023arxiv} uses these polynomial zonotope over approximations to solve the following optimization problem:
\begin{align}
    \label{eq:pz_optcost}
    &\underset{k \in K}{\min} &&\texttt{cost}(k) \\
    &&& \pzqjki \subseteq [\qlim^-, \qlim^+] \qquad &\forall j\in\Nq,i\in \Nt\\
    &&& \pzqdjki \subseteq [\dqlim^-, \dqlim^+] \qquad &\forall j\in\Nq,i\in \Nt \\
    &&& \label{eq:pz_rdf-col-avoid}
    \rdf(O_{\ell}, \pzFOki)> 0 \qquad &\forall \ell\in\NObs,i\in \Nt.
\end{align}
This formulation of the trajectory optimization problem has the benefit of being implementable without sacrificing any safety requirements.
In fact, as shown in \cite[Lemma 22]{armour2023arxiv}, any feasible solution to this optimization problem can be applied to generate motion that is collision free.
Though this method can be applied to $7$ degree of freedom systems in real-time, applying this method to perform real-time planning for more complex systems is challenging as we show in Sec. \ref{sec:results}. 
\section{Modeling \methodname{} with Neural Networks}
\label{sec:NN}

This section presents \methodname{}, a neural implicit representation that can encode obstacle-avoidance constraints in continuous-time.
In particular, \methodname{} predicts the distance between obstacles and the entire \emph{reachable set} of a robotic arm.
To construct this neural implicit representation, we require training data.
Unfortunately  computing the exact distance to the reachable set of multi-link articulated robotic arm is intractable because that multi-link arm is a non-convex set.
To build this training data, we rely on the polynomial zonotope-based representations presented in the previous section.

\begin{figure*}[t]
    \centering
    \includegraphics[width=\textwidth]{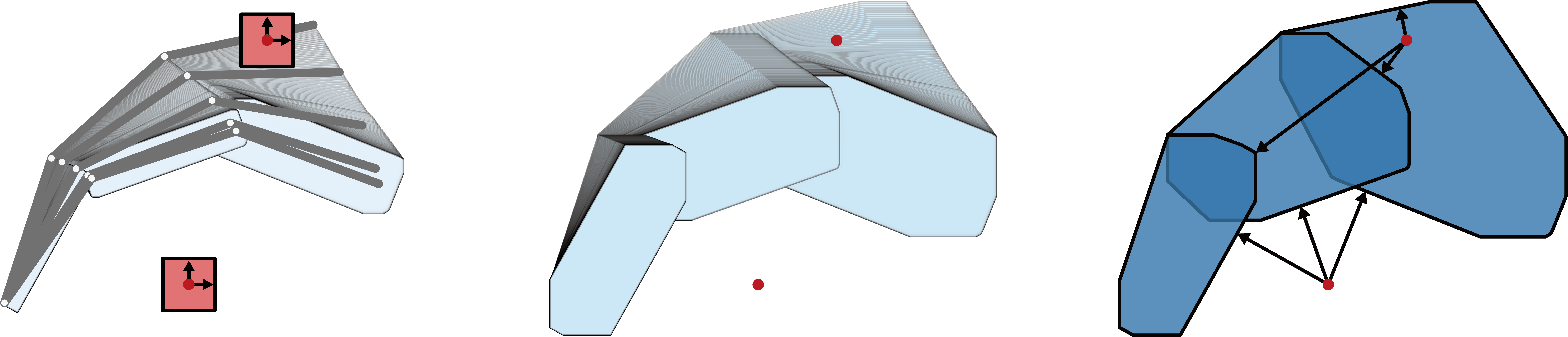}
    \caption{Visual illustration of Alg. \ref{alg:rdf}. (Left) Forward occupancy reachable set as a function of some initial conditions (Alg. \ref{alg:rdf}, line~\ref{lst:line:fo}). (Middle) Forward occupancy reachable set buffered by the obstacle generators (Alg. \ref{alg:rdf}, line~\ref{lst:line:buffered_fo}). (Right) Convex hull of each link's forward occupancy reach set (Alg. \ref{alg:rdf}, line~\ref{lst:line:conv}).}
    \label{fig:dataset}
    \vspace*{-0.2cm}
\end{figure*}


Importantly, we show that we can conservatively approximate the distance between an obstacle and sliced polynomial zonotope-based representation as the solution to a convex program.
This allows us to efficiently generate the training data required to construct our neural implicit representation.
Subsequently, we give an overview of the neural network architecture and loss function used for training.
Finally, we describe how to reformulate the trajectory optimization problem using the neural network representation of the reachability-based signed distance function.

\subsection{Derivation of RDF Approximation}\label{subsec:rdf_approx}
This subsection derives an approximation to the reachability-based signed distance function defined in Def. \ref{def:rdf_point}.
The core idea is to approximate the distance between an obstacle and the polynomial zonotope forward occupancy $\pz{FO}(\pzqi)$ \eqref{eq:pz_forward_occupancy} over of time for an entire trajectory.
Note that slicing a polynomial zonotope of all of its \emph{dependent} coefficients results in a zonotope \cite{kochdumper2020sparse}.  
This allows \methodname{} to approximate both positive and negative (\emph{i.e.} signed) distances by leveraging the zonotope arithmetic described in \ref{subsec:sdf_zono}.
We now present the main theorem of the paper whose proof can be found in supplementary material Appendix \ref{app:theorem}:

\begin{thm}
    Suppose a robot is following a parameterized trajectory $q(t;k)$ for all $t \in T$. 
    Consider an obstacle $O$ with center $\co$ and generators $\Go$ and $\pzFOjki$ with center $c_F$ and generators $G_F$.
    Let $\Pj := \bigcup_{i = 1}^{\nt} \pzFOjkibuf$ where $\pzFOjkibuf = (c_F, \Go  \cup \Gf)$. 
    Define the function $\ardf_j$ as follows:
    \begin{equation}
    \label{eq:ardf_j}
    \ardf_j(\co, \Pj) = 
        \begin{cases}
          \dist(\co, \partial\Pj)  & \text{if } \co \not \in \mathcal{P}  \\
          -\dist(\co, \partial\Pj) & \text{otherwise},
        \end{cases} 
    \end{equation}
    
    and define the function $\ardf$ as follows:
    \begin{equation}
    \label{eq:ardf}
        \ardf(\co, \cup_{j\in \Nq} \Pj) = \min_{j \in \Nq}  \ardf_j(\co, \Pj).
    \end{equation}
    If $\pzFOjki \cap O \neq \emptyset$, then $\ardf(\co, \Pj) \geq  \rdf(O, \FO(q(T;k)))$.
    If $\pzFOjki \cap O = \emptyset$, then $\ardf(\co, \Pj) \leq \rdf(O, \FO(q(T;k)))$.
\end{thm}
This theorem is useful because it allows us to conservatively approximate $\rdf$ using $\ardf$ which computes the distance between a point and a convex set. 
As we show next, this distance computation can be done by solving a convex program.
Recall that a convex hull can be represented as the intersection of a finite number of half planes, \emph{i.e.,}
\begin{align}
    \Pj=\bigcap_{h\in\Nhj}\Hjh 
\end{align}
where $\Nhj=\{1,\cdots,\nhj\}$ and $\nhj$ is the number of half-spaces \cite[Ch. 1]{z-lop-95}.
As a result, one can determine whether a point  $p\in\R^{\nd}$, is in side of $\Pj$ using the following property \cite[Ch. 1]{z-lop-95}:
\begin{align}
    p\in\Pj \Longleftrightarrow \Ajh\,p-\bjh\leq0\; \quad \forall i=\Nhj,
\end{align}  
where $\Ajh\in\R^{\nd},\,\bjh\in\R$ represents each half-space, $\Hjh$. 
As a result, one can compute the distance in \eqref{eq:ardf} as:
\begin{align}
    &\dist(\co,\partial\Pj)= \underset{p}{\min} &&||p-\co|| \\
    &&& \Ajh\,p-\bjh\leq0\qquad\forall h\in\Nhj,
\end{align}
where depending upon the norm chosen in the cost function one can solve the optimization problem using a linear or a convex quadratic program.
In the instance that there is non-trivial intersection between an obstacle and $\Pj$, one can apply the Euclidean projection \cite[p.398]{boyd2004convex} to directly calculate the distance between $\co$ and every half-space $\Hjh$ supporting $\Pj$:
\begin{align}
    \dist(\co,\partial\Pj) &= \min_{h\in\Nhj}\dist(\co,\Hjh)\\
    &= -\max_{h\in\Nhj}\frac{\Ajh\,\co-\bjh}{||\Ajh||}.
\end{align}
\begin{algorithm}[t]
\begin{algorithmic}[1]
\State $\{\pzqi\,:\, i \in \Nt\} \leftarrow (q_0,\dot q_0)$ 
\State $\co, \Go \leftarrow O$ 
\State{\bf for} $i = 1:n_t$ 
    \State\hspace{0.2in}{\bf for} $j = 1:n_q$ 
        \State\hspace{0.4in} $\pzFOjKi \leftarrow \texttt{PZFO}(\pzqi)$ 
        \State\hspace{0.4in} $\pzFOjki \leftarrow \texttt{slice}(\pzFOjKi, k)$ \label{lst:line:fo}
        \State\hspace{0.4in} $\pzFOjkibuf \leftarrow \pzFOjki\oplus Z_O^g$ \label{lst:line:buffered_fo}
    \State\hspace{0.2in}{\bf end for}
\State{\bf end for}
\State{\bf for} $j = 1:n_q$
    \State\hspace{0.2in} $\Pj \leftarrow \conv(\bigcup_{i = 1}^{\nt}\pzFOjkibuf)$ \label{lst:line:conv}
    \State\hspace{0.2in} {\bf if} $c_O \in \Pj$ {\bf then}
        \State\hspace{0.4in} $\ardf_j \leftarrow -\dist(\co,\partial\Pj)$ // negative distance 
    \State\hspace{0.2in} {\bf else}
        \State\hspace{0.4in} $\ardf_j \leftarrow \dist(\co,\partial\Pj)$ // positive distance 
    \State\hspace{0.2in} {\bf end if} 
\State{\bf end for}
\State \textbf{Return} $\{\ardf_j\,:\, j \in \Nq\}$
\end{algorithmic}
\caption{\texttt{RDF}$(q_0, \dot q_0, k, O)$}
\label{alg:rdf}
\end{algorithm}

\begin{figure*}[t]
    \centering
    \includegraphics[width=\textwidth]{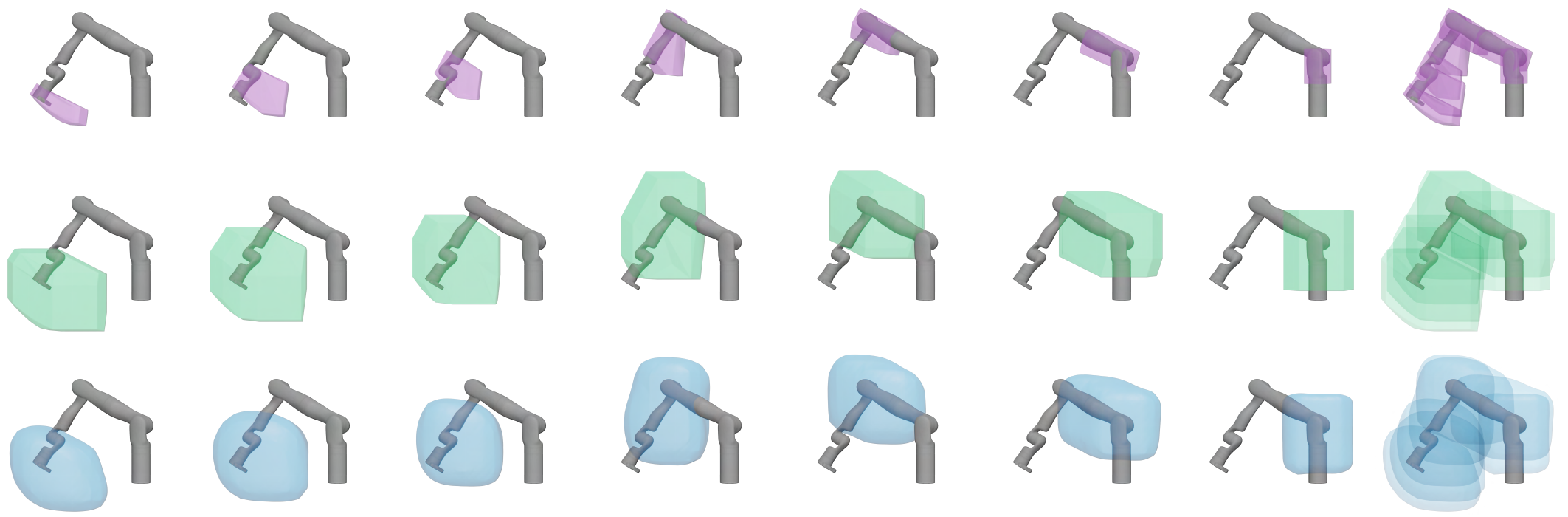}
    \caption{3D Reconstruction of \methodname{}'s zero-level sets compared to Alg. \ref{alg:rdf}. (Top) Convex hull of each link's forward occupancy reach set (purple) before buffering by obstacle generators (Alg. \ref{alg:rdf}, line~\ref{lst:line:fo}). (Middle) Following buffering, the size of the new forward occupancy (green) is increased (Alg. \ref{alg:rdf}, line~\ref{lst:line:buffered_fo}). (Bottom) \methodname{}'s zero-level set (blue) is shown to be a smooth approximation to that of the convex hull buffered forward occupancy reach set (middle, green). }
    \label{fig:zeroset}
    \vspace*{-0.2cm}
\end{figure*}
\subsection{Model Architecture}
We follow the \methodname{} model design of \cite{gropp2020implicit} and use an Multi-Layer Pereceptron (MLP) with 8 hidden layers and a jump connection in the middle between the $4$\ts{th} and $5$\ts{th} hidden layers. 
The network takes as input $x = (q_0, \dot{q}_0, k, \boldsymbol{c_o}) \subset \R^{3\nq + \nd}$, which consists of a concatenation of vectors corresponding to the initial joint positions $q_0$, initial joint velocities $\dot{q}_0$, trajectory parameters $k$, and obstacle center $\co$.
As described in Section \ref{subsec:trajectory} each $q_0, \dot{q}_0, k$ corresponds to particular desired trajectory and reachable set $\FO \subset \R^{\nd}$ that may or may not be in collision.
The network outputs an approximation $\hat{y} = (\hat{r}_1, \hat{r}_2, \cdots, \hat{r}_{n_q})$ of the reachability-based signed distance between the obstacle center $\co$ and forward occupancy of each link $\FO_j$.

\subsection{Loss Function}
\label{subsec:loss}
This section describes the loss function used to train the RDF model. 
We apply a mean square error loss added to an Eikonal loss term similar to \cite{gropp2020implicit}.
The mean square error loss forces the network to learn to predict the distance while the Eikonal loss regularizes the gradient of the RDF prediction. 
Given an input, ground-truth RDF distance pair $x = (q, \dot q, k, \boldsymbol{c_o})$, $y = (\ardf_1, \ardf_2, \cdots, \ardf_{\nq})$ from a batched dataset sample $(X_{batch},Y_{batch})$, our network, parameterized by its weights $\theta$, computes the output batch $\hat{Y}_{batch}=\{\hat{y} | \hat{y}=f_\theta(x), x \in X_{batch}\}$ and results in the loss:
\begin{equation}
    \mathcal{L} = \mathcal{L}_{MSE} + \alpha \cdot \mathcal{L}_{Eikonal}
\end{equation}
where 
\begin{equation}
    \mathcal{L}_{MSE} = \frac{1}{|\hat{Y}|} \sum_{\hat{y} \in \hat{Y}} (\frac{1}{\nq} \sum_{i=1}^{nq} (\hat{r}_i - \Tilde{r}_i)^2)
\end{equation}
\begin{equation}
    \mathcal{L}_{Eikonal} = \frac{1}{|\hat{Y}|} \sum_{\hat{y} \in \hat{Y}} (\frac{1}{\nq} \sum_{i=1}^{\nq}(\|\nabla_{\boldsymbol{c_O}} \hat{r}_i\|-1)^2),
\end{equation}
while $\alpha$ is a hyperparameter that denotes the coefficient of Eikonal loss $\mathcal{L}_{Eikonal}$ used in the total loss $\mathcal{L}$. 

\subsection{RDF-based Trajectory Optimization}
\label{subsec:rdf_planning}

After training, we generate a model $\rdfnn$ that takes in $(q_0, \dot{q}_0, k, c_{O,\ell})$ and predicts the reachability based distance between the obstacle and the robot's forward occupancy. 
Using this representation, we can reformulate the motion planning optimization problem described by \eqref{eq:pz_optcost}--\eqref{eq:pz_rdf-col-avoid} into:
\begin{align}
    \label{eq:rdf-nn-optcost}
    &\underset{k \in K}{\min} &&\texttt{cost}(k) \\
    &&&  \pzqjki \subseteq [\qlim^-, \qlim^+] \qquad &\forall j\in\Nq,i\in \Nt\\
    &&& \pzqdjki \subseteq [\dqlim^-, \dqlim^+] \qquad &\forall j\in\Nq,t\in \Nt \\
    &&& \label{eq:rdf-nn-col-avoid}
    \rdfnn(q_0, \dot q_0, k, c_{O,\ell})>\delta \qquad &\forall \ell\in\NObs
\end{align}
where $\delta$ is a buffer threshold of the RDF collision-avoidance constraint, equation \eqref{eq:rdf-nn-col-avoid}.
Note in particular that one computes the gradients of the last constraint by performing back-propagation through the neural network representation. 
The gradient of the first two constraints can be computed by applying \cite[Section IX.D]{armour2023arxiv}.

\section{\methodname{} Experimental Setup}\label{sec:experimental_setup}
This section describes our experimental setup including simulation environments, how the training and test sets were generated, and how the network hyperparameters were selected.

\subsection{Implementation Details}
We use Gurobi's  quadratic programming solver \cite{gurobi} to construct the ground truth reachability-based positive distance in Alg. \ref{alg:rdf}. 
The \methodname{} model is built and trained with Pytorch \cite{paszke2017automatic}. 
In the motion planning experiment, we ran trajectory optimization with IPOPT \cite{ipopt-cite}. 
A computer with 12 Intel(R) Core(TM) i7-6800K CPU @ 3.40GHz and an NVIDIA RTX A6000 GPU was used for 
experiments in Sec. \labelcref{subsec:rdf_accuracy_runtime,subsec:sdf_accuracy_runtime}. 
A computer with 12 Intel(R) Core(TM) i7-8700K CPU @ 3.70GHz and an NVIDIA RTX A6000 GPU was used for the motion planning experiment in Sec. \labelcref{subsec:exp_planning}.

\subsection{Simulation and Simulation Environment}
Each simulation environment has dimensions characterized by the closed interval $[-1, 1]^{\nd}$ and the base of the robot arm is located at the origin.
For each 2D environment, every link of the robotic arm is of same size and is adjusted according to the number of links $\nq$ to fit into the space.
We consider planar robot arms with 2, 4, 6, 8, and 10 links, respectively.
In each environment all obstacles are static, axis-aligned, and fixed-size where each side has length $\frac{0.2}{1.2 \nq}$.
For example, the 2D 6-DOF arm has a link length of 0.139m while obstacles are squares with side-length 0.028m.
In 3D we use the Kinova Gen3 7-DOF serial manipulator \cite{kinova-user-guide}.
The volume of each Kinova's link is represented as the smallest bounding box enclosing the native link geometry.

\begin{figure}[t]
    \centering
    \includegraphics[width=0.93\columnwidth]{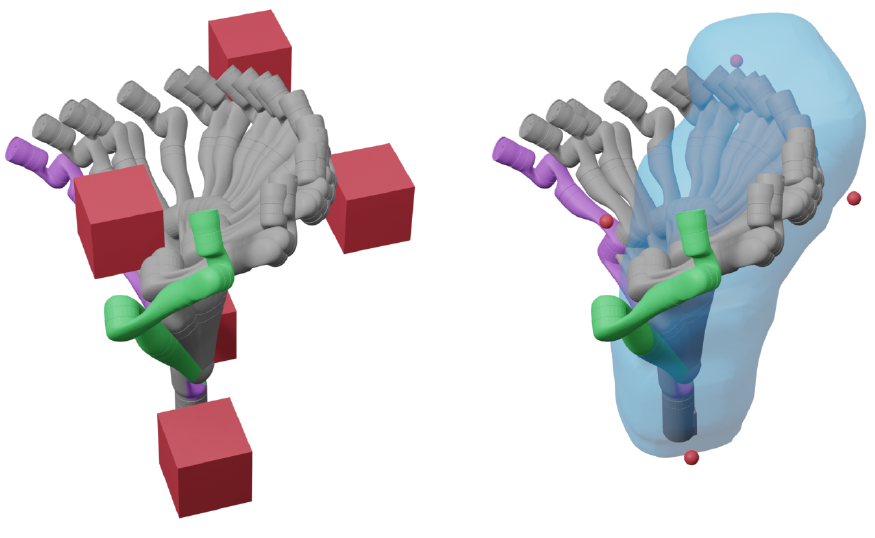}
    \caption{Real-time receding-horizon trajectory planning with \methodname{} in a cluttered environment. (Left) The arm safely moves from the start pose (purple) through intermediate configurations (grey) to reach the goal (green) while avoiding obstacles (red cubes). (Right) During the highlighted planning iteration all obstacle centers (red spheres) remain outside \methodname{}'s zero-level set (blue).}
    \label{fig:motion_planning}
    \vspace*{-0.3cm}
\end{figure}

We also ensure that each robot's initial configuration is feasible and does not exceed its position limits.
Each planning trial is considered a success if the l2-distance between the arm's configuration and goal configuration is within 0.1 rad.
A planning trial is considered a failure if the robot arms collides with an obstacle, if the trajectory planner fails to find a feasible trajectory in two successive steps, or if the robot does not reach the goal within 400 planning steps.

\subsection{Desired Trajectory}
We parameterize our trajectory with piece-wise linear velocity. 
We design \defemph{trajectory parameter} $k=(k_1,\cdots,k_{\nq})\in\R^{\nq}$ as a constant acceleration over $[0, t\plan)$. 
Then, the rest of the trajectory takes a constant braking acceleration over $[t\plan, \tfin]$ to reach stationary at $\tfin$.
Given an initial velocity $\dot{q}_{0}$, the parameterized trajectory is given by
\begin{align}
\label{eq:traj_parameterization}
    \dot{q}(t;k) = \begin{cases}
        \dot{q}_0 + k\,t, &t \in [0,t\plan) \\
        \frac{\dot{q}_0 + k\,t\plan}{\tfin - t\plan}(\tfin - t), &t \in [t\plan,\tfin].
    \end{cases}
\end{align}

\subsection{Dataset}\label{subsec:experimental_setup_dataset}

We compute the dataset for  \methodname{} by randomly sampling data points consisting of the initial joint position $q_0$, initial joint velocity $\dot q_0$, and trajectory parameter $k$.
For each initial condition, we then randomly sample $n_o=16$ obstacle center positions in $[-1, 1]^{\nd}$ and compute the ground truth distances between the forward reachable set of the robot and obstacles using Alg. \ref{alg:rdf}. 
The input to the network is then $x = (q_0, \dot q_0, k, \co)$ where $(q_0, \dot q_0, k)$ specifies the desired trajectory and the reachable set and $\co$ defines position of the center of the obstacle. 
The corresponding label is $y = (\ardf_1, \ardf_2, \cdots, \ardf_{\nq})$ where $\ardf_j$ is the the approximation of reachability-based signed distance to each link outlined in Alg. \ref{alg:rdf}.
For the 2D tasks, the datasets consist of 2.56 million samples while the 3D datasets consist of 5.12 million sample.
80\% of samples in each case are used for training and 20\% are used for validation. 
Another set of the same size as the validation set is generated for testing.



\subsection{Network Hyperparameters}

We train models using all combinations of the following hyperparameters: learning rates $lr=(0.001, 0.0001)$, Eikonal loss coefficients $\alpha = (0.0, 0.001, 0.0001)$, and $\beta=(0.9, 0.999)$ and weight decay $\gamma = 0.01$ for the Adam optimizer.
We train the 2D and 3D models for 300 and 350 epochs, respectively. 
The model that performs best on the validation set is chosen for further evaluation.

\section{Results}\label{sec:results}

This section evaluates the performance of the trained \methodname{} network in terms of its accuracy, inference time, the time required to compute its gradient, and its ability to safely solve motion planning tasks.
We compare \methodname{}'s safety and success rate on motion planning tasks to ARMTD \cite{holmes2020armtd}, CHOMP \cite{zucker2013chomp}, and the method presented in \cite{koptev_implicit_nodate}.

\subsection{\methodname{} Accuracy and Runtime Compared to \QP}\label{subsec:rdf_accuracy_runtime}

This section compares \methodname{}'s distance prediction accuracy to the distances computed by the \QP.
We perform these comparisons for 2D planar multi-link robot arms and the Kinova Gen3 7DOF arm on the test sets that were not used to either train or validate \methodname{}.
As shown in Table \ref{tab:exp_accuracy}, each model has a mean prediction error of $<1$cm in the $l1$-norm.
These results are supported by Fig. \ref{fig:zeroset}, which shows \methodname{}'s zero-level sets are smooth approximations to Alg.1.

\begin{table}[t]
\centering
    \begin{tabular}{ | c | c | c |}
    \hline
    Env. Dim. & DOF & Mean Error (cm) $\downarrow$ \\ \hline
    \multirow{5}{*}{2} & 2 & 0.16 $\pm$ 0.15 \\\cline{2-3} 
     & 4 & 0.26 $\pm$ 0.30 \\\cline{2-3}
     & 6 & 0.37 $\pm$ 0.36 \\\cline{2-3}
     & 8 & 0.39 $\pm$ 0.48 \\\cline{2-3}
     & 10 & 0.51 $\pm$ 0.52 \\\hline 
     3 & 7 & 0.45 $\pm$ 0.48 \\ \hline   
    \end{tabular}
\caption{Mean $l1$-norm error of each RDF model evaluated on test set }
\label{tab:exp_accuracy}
\vspace*{-0.5cm}
\end{table}

We then compared the mean runtime of \methodname{}'s inference and gradient computations to the computation time of \QP and its first-order numerical gradient. 
These comparisons are done over a random sample of 1000 feasible data points $(q_0, \dot{q}_0, k, c_o)$.
As shown in Table \ref{tab:exp_runtime}, \methodname{} computes both distances and gradients at least an order of magnitude faster than \QP.
This result holds even when considering only the time required to solve the quadratic program in \QP.
Note also that \methodname{}'s runtime appears to grow linearly with the DOF of the system, while \QP's grows quadratically.

\begin{table*}[t]
\centering
\resizebox{\textwidth}{!}{
    \begin{tabular}{ | c | c | c | c | c | c | c | c | c |}
    \hline
    Env. Dim. & DOF & Alg.\ref{alg:rdf} Distance & Alg.\ref{alg:rdf} Gradient & QP Distance & QP Gradient & RDF Distance & RDF Gradient  \\ \hline
    \multirow{5}{*}{2} 
       & 2 & 0.08 $\pm$ 0.01 & 0.16 $\pm$ 0.01 & 0.009 $\pm$ 0.006 & 0.015 $\pm$ 0.005 & \textbf{0.0008 $\pm$ 0.00003} & \textbf{0.003 $\pm$ 0.0003} \\\cline{2-8} 
       & 4  & 0.18 $\pm$ 0.01 & 0.70 $\pm$ 0.03 & 0.020 $\pm$ 0.011 & 0.065 $\pm$ 0.021 &  \textbf{0.0009 $\pm$ 0.00004} & \textbf{0.005 $\pm$ 0.0008}          \\\cline{2-8}
       & 6  & 0.28 $\pm$ 0.01 & 1.65 $\pm$ 0.05 & 0.028 $\pm$ 0.012 & 0.132 $\pm$ 0.033 & \textbf{0.0009 $\pm$ 0.00004} & \textbf{0.006 $\pm$ 0.0011}          \\\cline{2-8}
       & 8  & 0.40 $\pm$ 0.02 & 3.19 $\pm$ 0.12 & 0.033 $\pm$ 0.015 & 0.238 $\pm$ 0.065 & \textbf{0.0010 $\pm$ 0.00005} & \textbf{0.007 $\pm$ 0.0016}           \\\cline{2-8}
       & 10 & 0.85 $\pm$ 0.02 & 8.53 $\pm$ 0.11 & 0.030 $\pm$ 0.008 & 0.300 $\pm$ 0.044 & \textbf{0.0010 $\pm$ 0.00018} & \textbf{0.008 $\pm$ 0.0019}  \\\hline 
     3 & 7  & 1.59 $\pm$ 0.11 & 11.1 $\pm$ 0.80 & 0.251 $\pm$ 0.117 & 1.755 $\pm$ 0.807 &  \textbf{0.0011 $\pm$ 0.00013} & \textbf{0.006 $\pm$ 0.0015}   \\\hline   
    \end{tabular}
}
\caption{Mean Runtime of Distance and Gradient Computation $\downarrow$}
\label{tab:exp_runtime}
\end{table*}

\subsection{Accuracy \& Runtime Comparison with SDF}\label{subsec:sdf_accuracy_runtime}
\label{subsec:rdf_vs_sdf}
We compared \methodname{}'s distance prediction accuracy and runtime to that of an SDF-based model similar to \cite{koptev_implicit_nodate} over an entire trajectory in 3D.
To train a discrete-time, SDF-based model similar to that presented in \cite{koptev_implicit_nodate}, we generated a dataset of 5.12 million examples.
Each input to this SDF takes the form $x = (q_d(t;k), \co)$ and the corresponding label is $y = (\asdf_1, \asdf_2, \cdots, \asdf_{\nq})$, where $\asdf_j$ is the distance between $\co$ and a polynomial zonotope over approximation of $j$\ts{th} link of the robot.
Note that, in principle, this is equivalent to evaluating \methodname{} at stationary configurations by specifying $\dot q_0 = 0$ and $k = 0$.
Following \cite{koptev_implicit_nodate}, we also ensure that the number of collision and non-collision samples are balanced for each link. 

For \methodname{}, we generated 1000 samples where the $i$\ts{th} sample is of the form $(q_0, \dot q_0, k, \co)_i$.
Because SDF is a discrete-time model, its corresponding $i$\ts{th} sample is the minimum distance between the obstacle and a set of robot configurations $\{q_d(t_n;k) :\ n \in \Nt \}$ sampled at timepoints $t_n$ evenly separated by a given $\Delta t$.
Note that for SDF, we considered multiple time discretizations ($\Delta t = 0.01s, 0.02, 0.1s, 0.5s, 1.0s$).
During the implementation of SDF, we allow the forward pass through the network to be batched and evaluate all time steps for a given discretization size, simultaneously. 
As shown in Table \ref{tab:sdf_swept_volume}, \methodname{} has lower mean and max $l1$-norm error compared to SDF.
Similarly, \methodname{} has a lower run time than SDF across all time discretizations.

\begin{table}[t]
        \centering
        \resizebox{\columnwidth}{!}{
            \begin{tabular}{ | l | c | c | c | }
            \hline
            \multicolumn{1}{|c|}{Method} & Mean Error (cm) $\downarrow$ & Max Error (cm) $\downarrow$ & Runtime (ms) $\downarrow$  \\ \hline
            RDF & \textbf{0.55 $\pm$ 0.71} & \textbf{11} & \textbf{0.75 $\pm$ 0.02}  \\ \hline
            SDF ($\Delta t = 1.0s$) & 7.16 $\pm$ 11.5 & 88 & 1.17 $\pm$ 0.03\\ \hline
            SDF ($\Delta t = 0.5s$) & 0.96 $\pm$ 1.57 & 38 & 1.18 $\pm$ 0.12 \\ \hline
            SDF ($\Delta t = 0.1s$) & 0.79 $\pm$ 0.93 & 17 & 1.15 $\pm$ 0.03  \\ \hline
            SDF ($\Delta t = 0.02s$) & 0.80 $\pm$ 0.93 & 17 & 1.11 $\pm$ 0.03  \\ \hline
            SDF ($\Delta t = 0.01s$) & 0.80 $\pm$ 0.93 & 17 & 1.12 $\pm$ 0.06  \\ \hline
            \end{tabular}
        }
        \caption{RDF vs. SDF Distance Predictions}
        \label{tab:sdf_swept_volume}
        \vspace*{-0.2cm}
\end{table}




\subsection{Receding Horizon Motion Planning}\label{subsec:exp_planning}
This subsection describes the application of \methodname{} to real-time motion planning and compares its performance to several state of the art algorithms.
We evaluate each method's performance on a reaching task where the robot arm is required to move from an initial configuration to a goal configuration while avoiding collision with obstacles and satisfying joint limits.
Note that the planner is allowed to perform receding-horizon trajectory planning.
We evaluate each planner's success rate, collision rate, and mean planning times under various planning time limits. 
If the planner was unable to find a safe solution, the arm will execute the fail-safe maneuver from the previous plan.

\subsubsection{2D Results}

In 2D, we compare the performance of \methodname{} to ARMTD \cite{holmes2020armtd} across a variety of different arms with varying degrees of freedom from 2-10 DOF to better understand the scalability of each approach.
In each instance, the robot is tasked with avoiding 2 obstacles and is evaluated over $500$ trials.
In the interest of simplicity, we select $\delta$ in \eqref{eq:rdf-nn-col-avoid} to be $3$cm to $3.5$cm which is approximately $10$ times larger than the mean RDF error prediction as described in Tab. \ref{tab:exp_accuracy}.
Because our goal is to develop a planning algorithm that can operate in real-time, we also evaluate the performance of these algorithms when each planning iteration is restricted to compute a solution within $5, 0.3$ and $0.033$s. 
Note that each planning algorithm can only be applied for $400$ planning iterations per trial. 

Tables \ref{tab:success2darmtd} and \ref{tab:success2drdf} summarize the results. 
Across all experiments, when both algorithms are given $5$s per planning iteration, ARMTD was always able to arrive at the goal more frequently than \methodname{}.
A similar pattern seems to hold as the number of degrees of freedom increase when each algorithm is given $0.3$s per planning iteration; however, in the 10 DOF case ARMTD's success rate drastically decreases while \methodname{}'s performance is mostly unaffected.
This is because the computation time of ARMTD grows dramatically as the number of DOFs increases as depicted in Table \ref{tab:time2d}.
This observation is more pronounced in the instance where both planning algorithms are only allowed to take $0.033$s per planning iteration. 
In that instance \methodname{}'s performance is unaffected as the number of DOFs increases while ARMTD is unable to succeed beyond the 2DOF case. 
Note across all experiments, none of the computed trajectories ended in collision.

\subsubsection{3D Results}
In 3D, we compare the performance of \methodname{} to ARMTD and an SDF-based version of the obstacle-avoidance constraints within a receding-horizon trajectory planning framework. 
Note that for \methodname{} and SDF, we buffer their distance predictions with buffer size 3cm, which is approximately five times larger than the mean prediction error in Tab. \ref{tab:sdf_swept_volume}.
We also compare the performance of each of the aforementioned methods in 3D to CHOMP \cite{zucker2013chomp}.
Because our goal is to develop a planning algorithm that can operate in real-time, we also evaluate the performance of these algorithms when each planning iteration is restricted to compute a solution within $5, 0.3$, $0.033$, and $0.025$s. 
We also consider the case of avoiding 5 obstacles and 10 obstacles. 
Each obstacle case was evaluated over $500$ trials.
Note that each planning algorithm can only be applied for $400$ planning iterations per trial.

Tables \ref{tab:success3d5obs} and \ref{tab:success3d10obs} summarize the results of the performance of each algorithm across different time limits for the 5 and 10 obstacles cases, respectively.
First observe that as the number of obstacles increases each algorithms performance decreases. 
Note in particular that for a fixed number of obstacles, the ability of each method to reach the goal decreases as the time limit on planning decreases. 
Though ARMTD initially performs the best, when the time limit is drastically reduced, RDF begins to perform better. 
This is because the computation time of ARMTD grows dramatically as the number of DOFs increases as depicted in Table \ref{tab:time3d}.
The transition from when ARMTD performs best to when RDF performs best occurs when the time limit is restricted to $0.033$s. 
Note that \methodname{}, ARMTD, and SDF are collision free across all tested trials.
However CHOMP has collisions in every instance where it is unable to reach the goal. 
An example of \methodname{} successfully planning around 5 obstacles is shown in Fig. \ref{fig:motion_planning}.

\begin{table}[t]
    \begin{minipage}[ht]{.45\columnwidth}
        \centering
            \begin{tabular}{ | c | c | c | c | }
                \hline 
                \multicolumn{4}{|c|}{ARMTD} \\\hline
                \multirow{2}{*}{DOF} & \multicolumn{3}{c|}{Time Limit}\\ \cline{2-4} 
                 & 5.0 & 0.3 & 0.033 \\\hline \hline
                2 & \textbf{354} & \textbf{354} & \textbf{338} \\\hline
                4 & \textbf{354} & \textbf{354} & 0 \\\hline
                6 & \textbf{358} & \textbf{356} & 0 \\\hline
                8 & \textbf{382} & \textbf{372} & 0 \\\hline
                10 & \textbf{380} & 6 & 0\\\hline
                \end{tabular}
        \caption{\# of successes for ARMTD in 2D planning with 5.0s, 0.3s, and 0.033s time limit $\uparrow$}
        \label{tab:success2darmtd}
    \end{minipage} 
    \hfill
    \begin{minipage}[ht]{.45\columnwidth}
    \centering
        \begin{tabular}{ | c | c | c | c | }
                \hline 
                \multicolumn{4}{|c|}{RDF} \\\hline
                \multirow{2}{*}{DOF} & \multicolumn{3}{c|}{Time Limit}\\ \cline{2-4} 
                 & 5.0 & 0.3 & 0.033 \\\hline \hline
                2 & 253 & 253 & 253 \\\hline
                4 & 246 & 245 & \textbf{243} \\\hline
                6 & 239 & 239 & \textbf{233} \\\hline
                8 & 239 & 239 & \textbf{235} \\\hline
                10 & 246 & \textbf{242} & \textbf{233} \\\hline
                \end{tabular}
        \caption{\# of successes for RDF in 2D planning with 5.0s, 0.3s, and 0.033s time limit $\uparrow$}
        \label{tab:success2drdf}
    \end{minipage}
\end{table}

\begin{table}[t]
        \centering
        \begin{tabular}{ | c | c | c |}
                \hline 
                DOF & RDF & ARMTD \\ \hline
                2 & \textbf{0.022 $\pm$ 0.241} & 0.034 $\pm$ 0.127 \\\hline
                4 & \textbf{0.023 $\pm$ 0.246} & 0.071 $\pm$ 0.191 \\\hline
                6 & \textbf{0.037 $\pm$ 0.310} & 0.108 $\pm$ 0.156 \\\hline
                8 & \textbf{0.023 $\pm$ 0.231} & 0.182 $\pm$ 0.166 \\\hline
                10 & \textbf{0.030 $\pm$ 0.272} &  0.417 $\pm$ 0.286\\\hline
                \end{tabular}

        \caption{Mean time for each planning step for RDF and ARMTD in 2D planning experiments with 5.0s time limit $\downarrow$}
        \label{tab:time2d}
\end{table}

\begin{table}[t]
        \centering
        \begin{tabular}{ | c | c | c | c | c |}
                \hline 
                \# Obstacles & \multicolumn{4}{c|}{5} \\ \hline
                time limit (s) & 5.0 & 0.3 & 0.033 & 0.025 \\\hline \hline
                RDF & 401 & 398 & 384 & 377  \\ \hline
                ARMTD & \textbf{487} & \textbf{479} & 0 & 0 \\ \hline 
                SDF & 421 & 420 & 339 & 304 \\ \hline 
                CHOMP & 401 & 392 & \textbf{386} & \textbf{386} \\ \hline
        \end{tabular}
        \caption{\# of successes for RDF, ARMTD, SDF, and CHOMP in 3D 7 link planning experiment with 5 obstacles $\uparrow$}
        \label{tab:success3d5obs}
\end{table}

\begin{table}[t]
        \centering
            \begin{tabular}{ | c | c | c | c | c |}
                \hline 
                \# Obstacles & \multicolumn{4}{c|}{10} \\ \hline
                time limit (s) & 5.0 & 0.3 & 0.033 & 0.025 \\\hline \hline
                RDF & 321 & 319 & \textbf{306} & \textbf{293}  \\ \hline
                ARMTD & \textbf{466} & 301 & 0 & 0 \\ \hline 
                SDF & 357 & \textbf{352} & 215 & 15 \\ \hline 
                CHOMP & 312 & 301 & 293 & \textbf{293} \\ \hline
                \end{tabular}
        \caption{\# of successes for RDF, ARMTD, SDF, and CHOMP in 3D 7 Link planning experiment with 10 obstacles $\uparrow$}
        \label{tab:success3d10obs}
\end{table}

\begin{table}[t]
        \centering
            \begin{tabular}{ | c | c | c | c | c |}
                \hline 
                \# Obstacles & \multicolumn{4}{c|}{5} \\ \hline
                time limit (s) & 5.0 & 0.3 & 0.033 & 0.025 \\\hline \hline
                RDF &  \textbf{0} & \textbf{0} & \textbf{0} & \textbf{0}  \\ \hline
                ARMTD & \textbf{0} & \textbf{0} & \textbf{0} & \textbf{0} \\ \hline 
                SDF & \textbf{0} & \textbf{0} & \textbf{0} & \textbf{0} \\ \hline 
                CHOMP & 99 & 108 & 114 & 114 \\ \hline
                \end{tabular}
        \caption{\# of collisions for RDF, ARMTD, SDF, and CHOMP in 3D 7 link planning experiment with 5 obstacles $\downarrow$}
        \label{tab:collision3d5obs}
\end{table}

\begin{table}[t]
        \centering
            \begin{tabular}{ | c | c | c | c | c |}
                \hline 
                \# Obstacles & \multicolumn{4}{c|}{10} \\ \hline
                time limit (s) & 5.0 & 0.3 & 0.033 & 0.025 \\\hline \hline
                RDF & \textbf{0} & \textbf{0} & \textbf{0} & \textbf{0} \\ \hline
                ARMTD & \textbf{0 }& \textbf{0} & \textbf{0} & \textbf{0} \\ \hline 
                SDF & \textbf{0} & \textbf{0} & \textbf{0} & \textbf{0} \\ \hline 
                CHOMP & 188 & 199 & 207 & 207 \\ \hline
                \end{tabular}
        \caption{\# of collisions for RDF, ARMTD, SDF, and CHOMP in 3D 7 link planning experiment with 10 obstacles $\downarrow$}
        \label{tab:collision3d10obs}
\end{table}


\begin{table}[!htbp]
        \centering
        
        \begin{tabular}{ | c | c | c |}
                \hline 
                \# Obstacles & 5 & 10 \\ \hline \hline
 
                RDF & \textbf{0.023 $\pm$ 0.183} & \textbf{0.037 $\pm$ 0.296}  \\ \hline
                ARMTD & 0.172 $\pm$ 0.131 & 0.289 $\pm$  0.358 \\ \hline 
                SDF & 0.038 $\pm$ 0.212 & 0.064 $\pm$ 0.354 \\ \hline 
                CHOMP & 0.086 $\pm$ 0.138 & 0.078 $\pm$ 0.211 \\ \hline
                \end{tabular}

        \caption{Mean time for each planning step for RDF, ARMTD, SDF, and CHOMP in the 3D 7 link planning experiment with 5.0s time limit $\downarrow$}
        \label{tab:time3d}
\end{table}

\section{Conclusion}
This paper introduces the Reachability-based signed Distance Function (\methodname{}), a neural implicit representation useful for safe robot motion planning.
We demonstrate \methodname{}'s viability as a collision-avoidance constraint within a real-time receding-horizon trajectory planning framework.
We show that \methodname{}'s distance computation is fast, accurate, and unlike model-based methods like ARMTD, scales linearly with the dimension of the system.
\methodname{} is also able to solve challenging motion planning tasks for high DOF robotic arms under limited planning horizons.

Future work will aim to improve \methodname{}'s properties. 
First, bounding the network's approximation error will ensure that \methodname{} can be used with guarantees on safety.
Second, better architecture design and additional implicit regularization will allow a single \methodname{} model to generalize to multiple robot morphologies.
Finally, we will aim to extend \methodname{} to handle dynamic obstacles.

\renewcommand{\bibfont}{\normalfont\footnotesize}
{\renewcommand{\markboth}[2]{}
\printbibliography}

\newpage
\appendices
\section{Proof of Theorem \ref{thm:zono_sdf}}
\label{app:theorem}

Before proving the required result, we prove several intermediate results. 
\subsection{Signed Distance Between Zonotopes}
\label{subsec:sdf_zono}

Here we derive the signed distance between two zonotopes.
To proceed we first require a condition for determining if two zonotopes are in collision:
\begin{lem}\label{lem:zono_intersection}
Let $Z_1 = (c_1, G_1)$ and $Z_2 = (c_2, G_2)$ be zonotopes. Then $Z_1 \cap Z_2 \neq \emptyset$ if and only if $c_1 \in Z_{2,\text{buf}} = (c_2, G_1 \cup G_2)$.
\end{lem}
\begin{proof}
Suppose the intersection of $Z_1$ and $Z_2$ is non-empty. This is equivalent to the existence of $z_1 \in Z_1$ and $z_2 \in Z_2$ such that $z_1 = z_2$. Furthermore, by definition of a zonotope (Defn. \ref{defn:zonotope}), this is equivalent to the existence of coefficients $\beta_1$ and $\beta_2$ such that $z_1 = c_1 + G_1 \beta_1$, $z_2 = c_2 + G_2 \beta_2$, and
\begin{align}
    &c_1 + G_1 \beta_1 = c_2 + G_2 \beta_2 \\
    \iff& c_1 = c_2 + G_2 \beta_2 - G_1 \beta_1 \\
    \iff& c_1 \in (c_1, G_1 \cup G_2) = Z_{2, \text{buf}}.
\end{align}
Thus we have shown that $Z_1 \cap Z_2 \neq \emptyset$ if and only if $c_1 \in (c_1, G_1 \cup G_2) = Z_{2, \text{buf}}$.
\end{proof}
\begin{figure}[t]
    \centering
    \includegraphics[width=0.8\columnwidth]{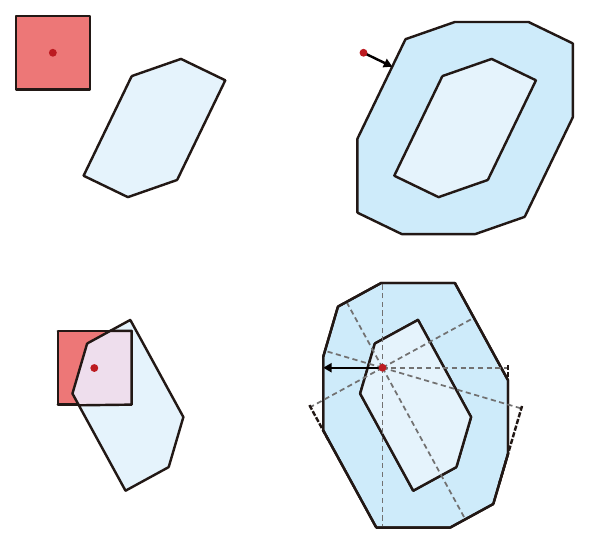}
    \caption{Illustration of signed distance for convex polytopes. In particular, this works well with zonotopes.} 
    \label{fig:zono_sdf}
    \vspace*{-0.2cm}
\end{figure}
Lemma \ref{lem:zono_intersection} can then be used to compute the the positive distance between two non-intersecting zonotopes:
\begin{lem}\label{lem:zono_positive_distance}
Let $Z_1 = (c_1, G_1)$ and $Z_2 = (c_2, G_2)$ be non-intersecting zonotopes.
\begin{align}
    \dist(Z_1, Z_2) = \dist(c_1; Z_{2,buf})
\end{align}
where $Z_{2,\text{buf}} = (c_2, G_1 \cup G_2)$.
\end{lem}
\begin{proof}
Suppose $Z_1$ and $Z_2$ are non-intersecting zonotopes such that $\dist(Z_1, Z_2) = r$ for some $r \in \mathbb{R}^{+}$. That is, there exists $z_1 \in Z_1$ and $z_2 \in Z_2$ such that 
\begin{align}
    r &= \|z_1 - z_2\| \\
    &= \|c_1 + G_1 \beta_1 - (c_2 + G_2 \beta_2) \| \\
    &= \| c_1 - (c_2 + G_2 \beta_2 = G_1 \beta_1) \| \\
    &\geq \dist(c_1, (c_2, G_1 \cup G_2)) \\
    &= \dist(c_1, Z_{2, \text{buf}}).
\end{align}
Now suppose $\dist(c_1, Z_{2, \text{buf}}) = r$. Then there exists coefficients $\beta_1$ and $beta_2$ such that
\begin{align}
    r &= \|c_1 - (c_2 + G_1 \beta_1 + G_2 \beta_2) \| \\
    &= \|c_1 - G_1 \beta_1 - (c_2 + G_2 \beta_2) \| \\
    &= \|z_1 - z_2\| \\
    &\geq \dist(Z_1, Z_2).
\end{align}
Together, these inequalities prove that $\dist(Z_1, Z_2) = \dist(c_1; Z_{2,buf})$.
\end{proof}
Let $\text{int}(S)$ denote the interior of a set.
Before deriving the negative distance we first define the penetration depth  \cite{cameron1986mtd} of two zonotopes:
\begin{defn}\label{defn:zono_pentration_depth}
    The \defemph{penetration depth} of $Z_1$ and $Z_2$ in $\R^{\nd}$ is denoted $\pi(Z_1, Z_2)$ and is defined by 
    \begin{equation}
        \pdf(Z_1, Z_2) = \min \{\|t\| :\ \text{int}(Z_1 + t \cap Z_2) = \emptyset, \quad t \in \R^{\nd}\}.
    \end{equation}
\end{defn}

The penetration depth $\pi(Z_1,Z_2)$ can be interpreted as the minimum translation applied to $Z_1$ such that the interior of $Z_1 \cap Z_2$ is empty.
Following \cite{agarwal2000convex} we can now redefine the penetration distance between $Z_1 = (c_1, G_1)$ and $Z_2 = (c_2, G_2)$ in terms of $c_1$ and $Z_{2,\text{buf}}$.
\begin{lem}\label{lem:zono_negative_distance}
Let $Z_1 = (c_1, G_1)$ and $Z_2 = (c_2, G_2)$ be zonotopes such that $Z_1 \cap Z_2 \neq \emptyset$.
Then the penetration depth between $Z_1$ and $Z_2$ is given by
   \begin{align}
     \pdf(Z_1, Z_2) = \pi(c_1, Z_{2,\text{buf}}),
\end{align} 
\end{lem}
\begin{proof}
Suppose $Z_1 \cap Z_2 \neq \emptyset$. 
Then by Defn. \ref{defn:zono_pentration_depth}, there exists a translation vector $t \in \R^{\nd}$ such that $\text{int}(Z_1 + t \cap Z_2) = \emptyset$. 
This means that $Z_1$ is translated by $t$ such that only the boundaries of $Z_1 + t$ and $Z_2$ are intersecting. 
Thus, there exists $z_1 \in \partial Z_1$ and $z_2 \in \partial Z_2$ such that $z_1 + t = z_2$. 
Furthermore, there exist coefficients $\beta_1$ and $\beta_2$ such that $z_1 = c_1 + G_1 \beta_1$, $z_2 = c_2 + G_2 \beta_2$, and
\begin{align}
    & c_1 + G_1 \beta_1 + t = c_2 + G_2 \beta_2 \\
    \iff& c_1 + t = c_2 + G_2 \beta_2 - G_1 \beta_1 \in \partial Z_{2,\text{buf}}.
\end{align}
Since $\pi(c_1, Z_{2,\text{buf}})$ is the minimum distance between $c_1$ and all points in $\partial Z_{2,\text{buf}}$, this means that $\pi(c_1, Z_{2,\text{buf}}) \leq \|t\|$.

Now suppose $c_1 \in Z_{2, \text{buf}}$.
Let $t \in \R^{\nd}$ such that $\|t\|$ is the penetration depth of $c_1$ and $Z_{2,\text{buf}}$. 
Then there exists coefficients $\beta_1$ and $\beta_2$ such that $z_2 = c_2 + G_1 \beta_1 + G_2 \beta_2$ and
\begin{align}
    c_1 + t = c_2 + G_1 \beta_1 + G_2 \beta_2 \\
    c_1 -G_1 \beta_1 + t = c_2 + G_2 \beta_2.
\end{align}
This means translating $z_1 = c_1 - G_1 \beta_1 \in \partial Z_1$ by $t$ causes the boundary of $Z_1$ to intersect the boundary of $Z_2$. 
Therefore $\pi(Z_1, Z_2) \leq \|t\|$.
Together the above inequalities prove the result.
\end{proof}

\begin{thm}\label{thm:zono_sdf}
Let $Z_1 = (c_1, G_1)$ and $Z_2 = (c_2, G_2)$ be zonotopes such that $Z_1 \cap Z_2 \neq \emptyset$. 
Then the signed distance between $Z_1$ and $Z_2$ is given by
\begin{equation}
    \sdf(Z_1, Z_2)  = 
        \begin{cases}
          \dist(c_1; Z_{2,buf})  & \text{if } c_1 \not \in Z_{2,buf} \\
          -\pdf(c_1, Z_{2,\text{buf}}) & \text{if } c_1 \in Z_{2,buf}
        \end{cases} .
\end{equation}
\end{thm}
\begin{proof} The proof follows readily from Lemmas \ref{lem:zono_positive_distance} and \ref{lem:zono_negative_distance}.
\end{proof}

Note that we provide a visual illustration of this theorem in Fig. \ref{fig:zono_sdf}.

\subsection{The Proof}

We prove the desired result by considering two cases. 
First assume  $O \cap \FO(q(T;k)) = \emptyset$ for all $t \in T$. 
Then
\begin{align}
    \rdf(O, \FO(q(T;k))) &= \underset{t \in T}{\min}\: \underset{j \in \Nq}{\min}\:  \sdf(O, \FO_j(q(t;k))) \\
    &\geq \underset{i \in \Nt}{\min}\: \underset{j \in \Nq}{\min}\:  \sdf(O, \pzFOjki) \\
    &= \underset{i \in \Nt}{\min}\: \underset{j \in \Nq}{\min}\: \sdf(\co, \pzFOjkibuf) \\
    &= \underset{j \in \Nq}{\min}\: \sdf(\co, \bigcup_{i\in\Nt} \pzFOjkibuf) \\
    &\geq \underset{j \in \Nq}{\min}\: \sdf(\co, \Pj) \\
    &= \underset{j \in \Nq}{\min}\: \ardf_j(\co, \Pj) \\
    &= \ardf(\co, \bigcup_{j\in \Nq} \Pj),
\end{align}
where the first equality follows from the definition of $\rdf$ in Def. \ref{def:rdf_point}, the second inequality follows from Lemma \ref{lem:pz_FO}, the third equality follows from Lemma \ref{lem:zono_positive_distance}, the fourth inequality by flipping the order of minimization, and the fifth inequality follows from Def. \ref{def:convhull}. 

Thus $\ardf$ underapproximates the distance between the forward occupancy and the obstacle. 
If instead $O \cap \FO_j(\q) \neq \emptyset$ for some $t \in T$, replacing $\geq$ with $\leq$ shows that we can overapproximate the penetration distance between an obstacle and the forward occupancy.


\end{document}